\documentclass[]{fairmeta}

\usepackage[utf8]{inputenc} 
\usepackage[T1]{fontenc}    
\usepackage{url}            
\usepackage{booktabs}       
\usepackage{caption}
\usepackage{graphicx}
\usepackage{amssymb}
\usepackage{amsmath}
\usepackage{amsthm}
\usepackage{amsfonts}       
\usepackage{nicefrac}       
\usepackage{microtype}      
\usepackage{algorithm}
\usepackage[endLComment=,italicComments=false]{algpseudocodex}
\usepackage{xcolor, colortbl}
\usepackage{tabularx}
\usepackage{multirow}
\usepackage{caption}
\usepackage{subcaption}
\usepackage{wrapfig}
\definecolor{lightcyan}{rgb}{0.88, 1, 1}
\usepackage[title]{appendix}
\usepackage{titlesec}
\usepackage{soul}  

\usepackage{tikz}
\usetikzlibrary{patterns}

\theoremstyle{plain}
\newtheorem{theorem}{Theorem}[section]
\newtheorem{proposition}[theorem]{Proposition}
\newtheorem{lemma}[theorem]{Lemma}
\newtheorem{corollary}[theorem]{Corollary}
\theoremstyle{definition}

\theoremstyle{remark}

\newlength\savewidth

\makeatletter

\newcommand{\customdashline}[3]{%
    \noalign{\vbox{\hrule height 0pt%
        \hbox to\linewidth{\cleaders\hbox to \dimexpr #1 + #2\relax{\hss\rule{#1}{#3}\hss}\hfill}%
    }}%
}
\makeatother

\renewcommand{\paragraph}[1]{\vspace{-.3em}\noindent\textbf{#1}}
\newcolumntype{x}[1]{>{\centering\arraybackslash}p{#1pt}}
\newcolumntype{y}[1]{>{\raggedright\arraybackslash}p{#1pt}}
\newcolumntype{z}[1]{>{\raggedleft\arraybackslash}p{#1pt}}
\newcommand{\app}{\raise.17ex\hbox{$\scriptstyle\sim$}}

\definecolor{deemph}{gray}{0.58}

\definecolor{baselinecolor}{gray}{.9}

\newcommand{\RNum}[1]{\uppercase\expandafter{\romannumeral #1\relax}}

\definecolor{lightcyan}{rgb}{0.88, 1, 1}
\definecolor{asparagus}{rgb}{0.53, 0.66, 0.42}
\definecolor{azure}{rgb}{0.0, 0.5, 1.0}
\definecolor{brightpink}{rgb}{1.0, 0.0, 0.5}
\definecolor{boston}{rgb}{0.8, 0.0, 0.0}
\definecolor{gray}{rgb}{0.75, 0.75, 0.75}
\definecolor{lightgray}{rgb}{0.88, 0.88, 0.88}
\definecolor{darkgray}{rgb}{0.50, 0.50, 0.50}
\definecolor{orange2}{rgb}{0.99, 0.86, 0.70}
\definecolor{pastelgray}{rgb}{0.81, 0.81, 0.77}
\definecolor{orangered}{rgb}{.99, .40, .00}
\definecolor{darkUT}{HTML}{BF5700}
\definecolor{lightUT}{HTML}{FFF1E6}
\definecolor{darkUT_B}{HTML}{005F86}
\definecolor{lightUT_B}{HTML}{F5FDFF}
\definecolor{darkUT_G}{HTML}{333F48}
\definecolor{lightUT_G}{HTML}{F5F4F0}
\definecolor{darkgreen}{HTML}{F5F4F0}
\definecolor{apricot}{HTML}{FFD7B8}

\usepackage[most]{tcolorbox}
\newtcolorbox{remarkbox}[1][]{
  enhanced,
  breakable,
  colback=utlight!10,        
  colframe=utlight,             
  colbacktitle=utmain!80!black,         
  coltitle=white,                 
  fonttitle=\bfseries,            
  title=#1,                       
  boxed title style={
    frame empty,
    left=2pt,
    right=2pt,
    top=2pt,
    bottom=2pt,
  },
  rounded corners,                
  boxrule=1pt,                    
  arc=1mm,                        
  before skip=1em,                
  after skip=1em,                 
}

\usepackage{tikz}

\newcounter{takeawayonly}

\newcounter{closing}

\definecolor{mlBg}{HTML}{FAFAFA}      
\definecolor{mlFg}{HTML}{37474F}      
\definecolor{mlKeyword}{HTML}{7C4DFF} 
\definecolor{mlConst}{HTML}{D81B60}   
\definecolor{mlString}{HTML}{91B859}  
\definecolor{mlComment}{HTML}{90A4AE} 
\definecolor{mlBuiltin}{HTML}{00BCD4}
\definecolor{mlLib}{HTML}{00897B}    
\definecolor{mlLineno}{HTML}{B0BEC5}

\usepackage[T1]{fontenc}
\usepackage[scaled=0.9]{inconsolata}
\newcommand{\pyfont}{\ttfamily}
\usepackage{listings}
\usepackage{caption}
\lstdefinestyle{py-material-light}{
  language=Python,
  backgroundcolor=\color{mlBg},
  basicstyle=\pyfont\small\color{mlFg},
  showstringspaces=false,
  breaklines=true,
  tabsize=4,
  keepspaces=true,
  columns=fullflexible,
  numbers=none, 
  commentstyle=\itshape\color{mlComment},
  stringstyle=\color{mlString},
  keywordstyle=\bfseries\color{mlKeyword},
  morekeywords=[2]{True,False,None},
  keywordstyle=[2]\bfseries\color{mlConst},
  emph={print,range,len,enumerate,zip,dict,list,set,tuple,int,float,str,bool,open,
        sorted,sum,min,max,any,all,abs,super,isinstance,type,property},
  emphstyle=\color{mlBuiltin},
  emph={[2]{np,pd,plt,torch,tf,jax,sklearn}},
  emphstyle=[2]\color{mlLib}
}

\definecolor{aoBg}{HTML}{FAFAFA}
\definecolor{aoFg}{HTML}{383A42}
\definecolor{aoComment}{HTML}{A0A1A7}

\usepackage[cjk]{kotex}
\usepackage[font=small]{subcaption}
\usepackage[font=footnotesize]{subcaption}

\usepackage{xcolor}
\usepackage{lipsum}
\usepackage{adjustbox}
\usepackage{minted} 

\usepackage{listings}

\title{Foundation-Preserving Optimization in Generalized Eigenspace}

\author[1]{Dongjun Kim}
\author[2]{Adrian de Wynter}
\author[3]{Huancheng Chen}
\author[4]{Heasung Kim}
\author[1]{Haris Vikalo}

\affiliation[1]{The University of Texas at Austin}
\affiliation[2]{Microsoft}
\affiliation[3]{Microsoft AI}
\affiliation[4]{Meta}

\abstract{
While fine-tuning effectively adapts foundation models to desired downstream behavior, it can lead to \textit{catastrophic forgetting}--that is, degrading foundation knowledge acquired during pretraining. We propose Optimization in Generalized Eigenspace (OGE), a forgetting-aware optimization framework for supervised finetuning. Guided by a first-order preservation condition, OGE defines a forgetting penalty over activations sampled from a proxy pretraining distribution and an adaptation utility over downstream task activations. It then constructs a generalized eigenspace coordinate in which the stability-plasticity trade-off can be optimized directly by solving a generalized eigenvalue problem derived from a Rayleigh quotient. This coordinate system enables direction-wise gated AdamW updates that attenuate directions with low utility relative to forgetting risk throughout training. To estimate the forgetting risks, OGE constructs pretraining-proxy calibration data by sampling model's generative output rather than relying on a narrow proxy set. Experiments on math, code, and instruction-following adaptation show that OGE achieves a strong preservation–adaptation balance over baselines.
}

\date{\today}

\metadata[Correspondence]{Dongjun Kim at \url{dongjungim20@utexas.edu}}

\begin{document}
\maketitle
\setcounter{tocdepth}{1}
\addtocontents{toc}{\protect\setcounter{tocdepth}{-1}}

\section{Introduction}
\label{intro}
\setlength\intextsep{0pt}
\begin{wrapfigure}{r}{8cm}
\centering
\includegraphics[width=0.5\textwidth]{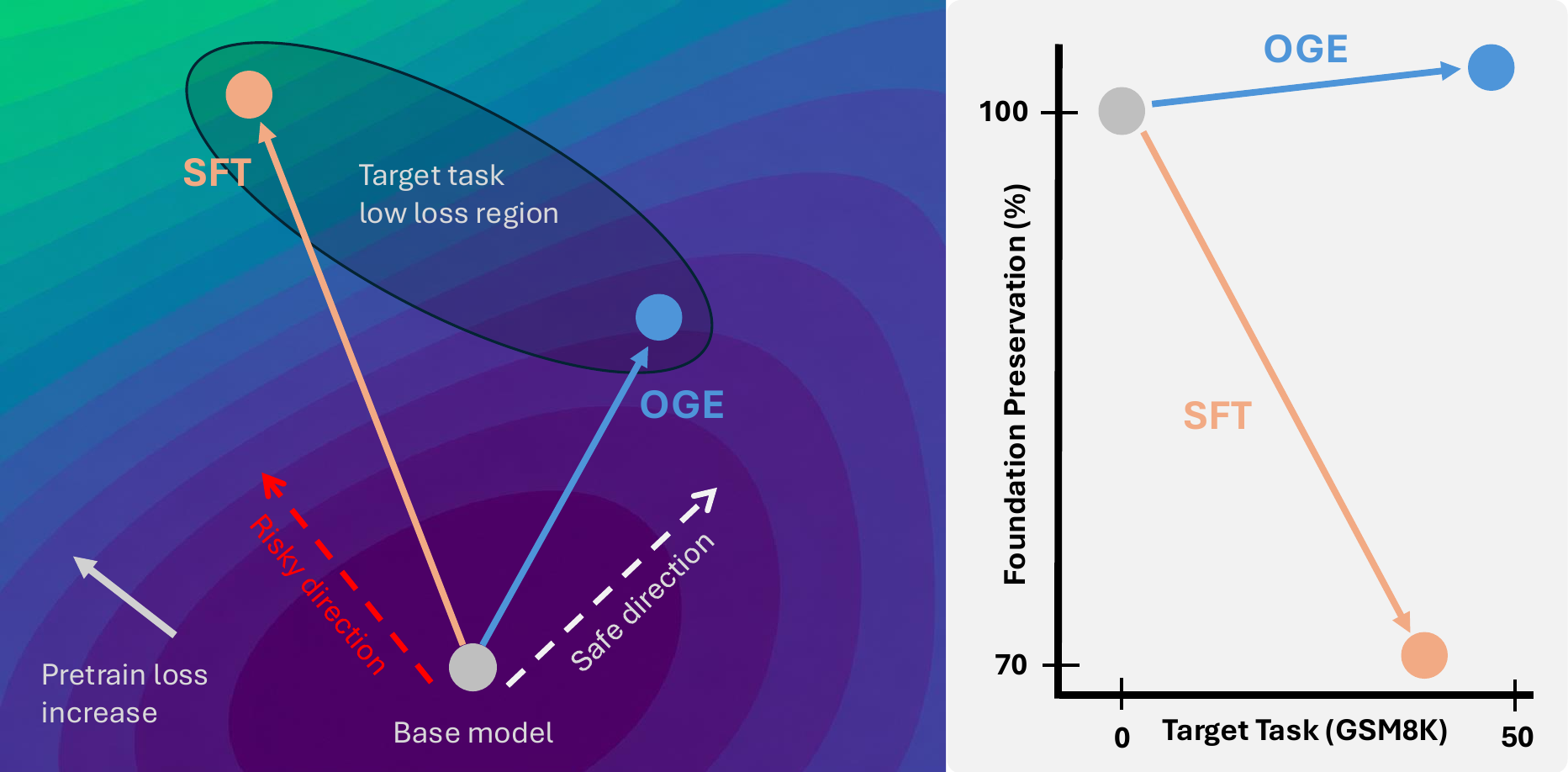}
\caption{\textit{Left}: Pretrain loss landscape. OGE optimization is regularized to suppress directions that pose high risk to foundation knowledge. \textit{Right}: OGE preserves foundation knowledge while achieving decent target task adaptation quality.}
\label{fig:overview}
\end{wrapfigure}
Foundation models have achieved remarkable success, powering a wide range of domains~\cite{llm1, llm2, instfollowing, preference}. In particular, post-training following large-scale pretraining has become the de facto paradigm for deploying language models~\cite{glm, kimi, v4}. Specifically, modern frontier models typically undergo multiple stages of post-training, such as supervised fine-tuning (SFT), reinforcement learning from human feedback (RLHF)~\cite{deepseek}, direct preference optimization (DPO)~\cite{dpo}, or on-policy distillation~\cite{opd} to align the model's behavior and endow domain-specific capabilities. However, such sequential optimization toward task-specific target distributions often leads to undesirable \textit{catastrophic forgetting} of previously acquired knowledge~\cite {mapping,forget2,forget3,forgetting5}. Despite its effectiveness and simplicity, SFT is particularly susceptible to forgetting because it directly optimizes the model toward the narrow target distribution, inducing abrupt shifts away from the pretrained parameters~\cite{rlsft}. For foundation models, this degradation is especially problematic: pretrained knowledge not only supports broad non-target capabilities across domains, but also serves as the substrate for adapting to target tasks. Effective adaptation should therefore improve the target task without disrupting pretrained behavior that supports generalization and reuse. 

    

The stability–plasticity dilemma is a long-standing challenge in neural network adaptation, referring to the tension between preserving previously acquired knowledge and effectively learning new tasks~\citep{dilema1,dilema2}. Several pioneering studies have sought to mitigate catastrophic forgetting through carefully chosen initialization strategies~\cite{milora, loranull, corda} or architectural modifications~\cite{oplora} of parameter-efficient fine-tuning (PEFT) methods. However, because forgetting is closely tied to optimization dynamics~\cite{dynamic}, these approaches may not adequately balance the stability–plasticity trade-off, often favoring preservation at the expense of adaptation or vice versa. This motivates our central question: \emph{Can fine-tuning optimization itself directly balance this trade-off?} To address this, we introduce an optimizer-level mechanism that explicitly balances adaptation utility and forgetting risk through direction-wise modulation of weight updates.

In this paper, we propose Optimization in Generalized Eigenspace (OGE), a forgetting-aware optimization framework for supervised finetuning. Starting from a first-order condition for preserving the pretrained behavior of a nonlinear block, OGE defines forgetting risk over pretraining-proxy data and an adaptation utility over downstream-task data. OGE scores candidate update directions by the ratio between adaptation utility and forgetting risk, expressed as a generalized Rayleigh quotient. Maximizing this quotient leads to a generalized eigenvalue problem whose solution defines a spectral coordinate system in which each basis receives an explicit utility-to-risk ratio/score. OGE then performs AdamW-style updates in this coordinate system with basis-wise spectral gating, attenuating directions whose adaptation utility is small relative to their forgetting risk as illustrated in  Fig.~\ref{fig:overview} (left). To estimate the forgetting risk, which requires activation statistics induced by pretraining data, OGE provides a principled approach to approximating the pretraining corpus distribution by sampling from the model’s own generative distribution. In this way, OGE achieves comparable or even better target task adaptation, while not compromising foundation knowledge (Fig.~\ref{fig:overview} (right)).

Our main contributions are: 
\begin{itemize}
    \item We formulate stability-plasticity tradeoff as a generalized Rayleigh quotient optimization problem that compares downstream adaptation utility with forgetting risk.
    \item We construct a generalized-eigenbasis for weight update directions from the activation-induced adaptation utility and forgetting risk, and introduce spectral-gated AdamW (SAdam) optimizer for direction-wise control of updates throughout training.
    \item We propose model-generated calibration, a procedure to construct pretraining-proxy calibration data from sequences generated by the pretrained model, avoiding reliance on an external narrow benchmark.
    \item We evaluate OGE across diverse tasks and benchmarks, showing a superior preservation-adaptation balance over baselines.
\end{itemize}

\section{Related Work}
\label{background}
\paragraph{Language Model Forgetting.} A growing line of work investigates forgetting of language models. \citet{deepseek} empirically validate that post-training degrades several foundation capability of the original base pretrained model. \citet{sharpness} found that forgetting is related tothe  sharpness of the loss landscape. \citet{mapping} empirically evaluate forgetting across domains and observe that factual retrieval abilities are often among the most severely degraded. On the one hand, recent studies suggest that on-policy adaptation methods are more robust to forgetting than their off-policy counterparts. For example, \citet{razor} show that reinforcement learning exhibits less forgetting than SFT, while SDFT~\citep{sdft} demonstrates that on-policy self-distillation can preserve pretrained capabilities. In this work, we center our study on SFT because of its broad applicability, practical effectiveness, and particular susceptibility to catastrophic forgetting.

Several methods seek to mitigate forgetting of SFT with low-rank adaptation (LoRA)~\cite{lora}. MiLoRA~\cite{milora} decomposes pretrained weight matrices by SVD and initializes LoRA adapters using minor singular components, starting training from the subspace that is less dominant in the pretrained weight. CorDA~\cite{corda} and LoRA-Null~\cite{loranull} construct such an initialization subspace using activation covariance matrices estimated from pretraining-proxy calibration data. These methods show that the geometry of initialized low-rank space is important to mitigate forgetting, but their preservation mechanisms are fixed at the start of initialization point. Once fine-tuning begins, they impose no further protection. Moreover, activation-based methods such as CorDA and LoRA-Null rely on a single, arbitrarily chosen dataset as a proxy for the pretraining distribution, limiting broad, generalizable protection. OPLoRA~\citep{oplora} applies an orthogonal projector to steer LoRA updates away from the principal components of the pretrained weights. However, such hard projections may also suppress directions that are beneficial for the target task, thereby limiting downstream adaptation. Moreover, recent work~\citep{locating} shows that minor singular components of pretrained weights can also encode foundational behaviors. Therefore, identifying foundation-critical directions solely from weight geometry, without accounting for activation statistics, as in MiLoRA and OPLoRA, can be misleading. OGE differs from prior approaches by making the optimization process itself jointly adaptation- and forgetting-aware. It identifies foundation-critical directions using activation statistics collected from a principled approximation of the pretraining distribution. It then derives update directions that maximize adaptation utility while minimizing forgetting risk, yielding a favorable preservation–adaptation trade-off.

\section{Method}
\label{method}
\begin{figure*}[!t]
    \centering
    \includegraphics[width=1.0\textwidth]{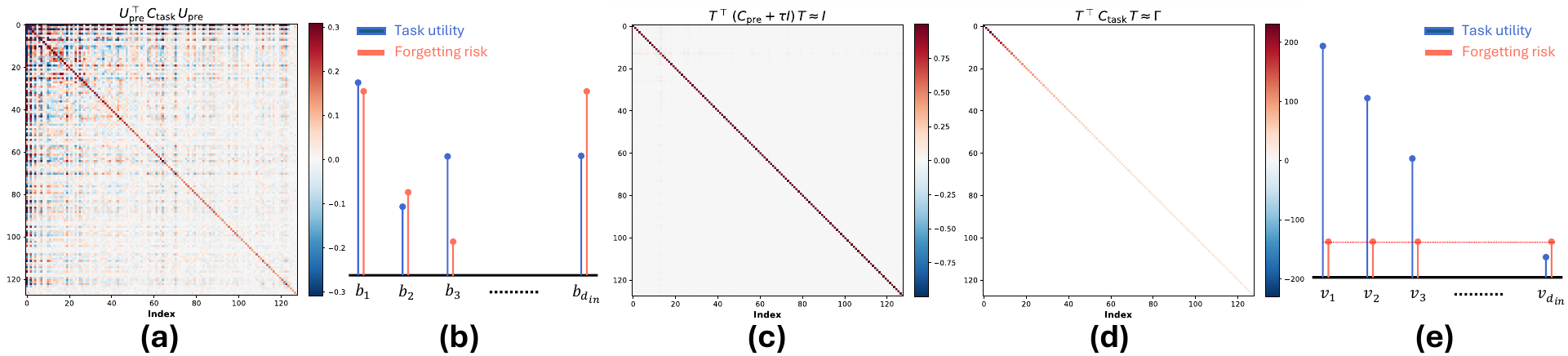}
    \caption{
    Illustration of the generalized eigenspace (Layer-15 MLP downproj). 
    (a) $U_{\mathrm{pre}}^\top C_{\mathrm{task}} U_{\mathrm{pre}}$ is not diagonal, indicating misalignment of geometric structure between pretraining and task covariance eigenspace. 
    (b) Example basis configuration. In the original Euclidean basis, adaptation utility and forgetting risk need not be aligned. 
    (c) The generalized basis whitens the forgetting risk, $V^\top C_{\mathrm{pre}} V=I$. 
    (d) The same basis diagonalizes adaptation utility, $V^\top C_{\mathrm{task}}V=\Gamma$. 
    (e) The resulting $C_\mathrm{pre}$-orthogonal generalized eigenspace coordinate ranks directions by utility-to-risk ratio, identifying safe update direction.}
    \label{fig:first_fig}
    
\end{figure*}
We first derive a first-order condition for preserving the pretrained behavior of a nonlinear layer in Sec.~\ref{sec:3.1}. Sec.~\ref{sec:3.2} uses this condition to define the generalized Rayleigh quotient via adaptation utility and forgetting risk. Optimizing this quotient induces a new coordinate system. Sec.~\ref{sec:3.3} introduces spectral-gated AdamW (SAdam), which applies basis-wise gating to AdamW updates in this coordinate system. Finally, Sec.~\ref{sec:selfseeded} describes a model-generated calibration set, where sequences sampled from the pretrained model serve as a proxy for the pretraining data distribution. 

\subsection{First-Order Condition for Preserving Pretrained Behavior}
\label{sec:3.1}
Let $X_{\rm pre}, X_{\rm task}\in\mathbb{R}^{d_{\rm in}\times T}$ denote the layer-input token sequences collected from the pretraining corpus and downstream task datasets, respectively.
Consider a non-linear block, 
\begin{equation}
    f_{W_0}(X_\mathrm{pre})=\sigma(W_0X_\mathrm{pre}+b\mathbf{1}_T^\top),
\end{equation}
where \(W_0\in\mathbb{R}^{d_{\mathrm{out}}\times d_{\mathrm{in}}}\), \(b\in\mathbb{R}^{d_{\mathrm{out}}}\), and \(\sigma\)  
is an element-wise non-linear mapping. Define $Z_0=W_0X_\mathrm{pre}+b\mathbf{1}_T^\top$. For a small weight perturbation \(\Delta W\), first-order Taylor expansion of \(\sigma\) around \(Z_0\) results in
\begin{equation*}
f_{W_0+\Delta W}(X_\mathrm{pre})
\approx
f_{W_0}(X_\mathrm{pre})+J_\sigma(Z_0)[\Delta W X_\mathrm{pre}], 
\end{equation*}
where \(J_\sigma(Z_0)\) denotes the Jacobian of \(\sigma\) at \(Z_0\). Therefore, the first-order output drift on activations from the pretraining distribution is controlled by \(\Delta W X_{\mathrm{pre}}\), and preserving pretrained behavior can be encouraged by penalizing the magnitude of this perturbation. This motivates the forgetting risk
\begin{equation*}
\scalebox{0.95}{$\displaystyle
\mathcal R_{\mathrm{pre}}(\Delta W)
:=
\|\Delta W X_{\mathrm{pre}}\|_F^2
=
\operatorname{Tr}(\Delta W C_{\mathrm{pre}}\Delta W^\top),
$}
\end{equation*}
where $C_{\mathrm{pre}}:=X_{\mathrm{pre}}X_{\mathrm{pre}}^\top$. Under the same first-order view, an update that is useful for downstream task adaptation should induce a large response on the activations. This motivates the adaptation utility
\begingroup
\fontsize{10pt}{11pt}\selectfont
\begin{equation*}
\mathcal U_{\mathrm{task}}(\Delta W)
:=
\|\Delta W X_{\mathrm{task}}\|_F^2
=
\operatorname{Tr}(\Delta W C_{\mathrm{task}}\Delta W^\top),
\end{equation*}
\endgroup
where $C_{\mathrm{task}}:=X_{\mathrm{task}}X_{\mathrm{task}}^\top$. We aim to maximize adaptation utility while minimizing forgetting risk. We formulate this relation with the generalized Rayleigh quotient~\cite{grc} as a ratio of adaptation utility over forgetting risk. To derive this objective, consider a rank-one perturbation of $W$,
\begin{equation*}
\Delta W = uv^\top,
\qquad
u\in\mathbb{R}^{d_{\mathrm{out}}},\quad v\in\mathbb{R}^{d_{\mathrm{in}}}.
\end{equation*}
Then, the generalized Rayleigh quotient becomes:
\begin{equation*}
\rho(v):=
\frac{\mathcal U_{\mathrm{task}}}{\mathcal R_{\mathrm{pre}}}
=\frac{\|u\|_2^2\,v^\top C_{\mathrm{task}}v}{\|u\|_2^2\,v^\top C_{\mathrm{pre}} v}
=\frac{v^\top C_{\mathrm{task}}v}{v^\top C_{\mathrm{pre}} v}.
\end{equation*}
We can observe that the quotient derived from the first-order preservation condition only depends on the input direction $v$. In Sec.~\ref{sec:3.2}, maximizing this Rayleigh quotient leads to a generalized eigenspace of input directions that are ranked by \(\rho(v)\), i.e., adaptation utility per unit forgetting risk.

\subsection{Generalized Eigenbasis for Forgetting-Aware Optimization}
\label{sec:3.2}
We estimate \(C_{\mathrm{pre}}\) and \(C_{\mathrm{task}}\) by collecting layer-input token sequence
from forward-pass with pretraining data \(D_{\mathrm{pre}}\) and target-task data
\(D_{\mathrm{task}}\), respectively, using the pretrained model. Since these
activation distributions differ, the corresponding covariances need not share eigenvectors.
To diagnose this, we express \(C_{\mathrm{task}}\) in the eigenbasis \(U_{\mathrm{pre}}\) of
\(C_{\mathrm{pre}}\). The off-diagonal mass in Fig.~\ref{fig:first_fig}(a) indicates mixing of pretraining modes by target-task variation. Therefore, the \(C_{\mathrm{pre}}\) eigenbasis diagonalizes the forgetting risk but leaves
adaptation utility coupled, while the \(C_{\mathrm{task}}\) eigenbasis does the opposite. Optimizing in either single-covariance
basis cannot capture this entangled geometry. As illustrated in
Fig.~\ref{fig:first_fig}(b), a direction with high adaptation utility (e.g., $b_1$) alone need not be optimal
once forgetting risk is considered. We therefore search for directions that directly
maximize \(\rho(v)\), the adaptation utility per unit forgetting risk.
Because $\rho(\alpha v)=\rho(v)$ for any $\alpha \neq 0$, scale-invariant problem becomes the constraint optimization 
\begin{equation}
\label{eq:optimization_problem}
\max_{v \in \mathbb{R}^{d_\mathrm{in}}} \; v^\top C_{\mathrm{task}} v
\qquad
\text{s.t.}
\qquad
v^\top C_{\mathrm{pre}} v = 1.
\end{equation}
Under the $C_{\mathrm{pre}} \succ 0$, Lagrangian of Eq.~\ref{eq:optimization_problem} becomes
\begin{equation*}
\mathcal{L}(v,\gamma)
=
v^\top C_{\mathrm{task}} v
-
\gamma \bigl(v^\top C_{\mathrm{pre}} v - 1\bigr).
\end{equation*}
Taking the derivative of the Lagrangian with respect to $v$ and setting it to zero yields the generalized eigenvalue equation
\begin{equation*}
C_{\mathrm{task}} v = \gamma C_{\mathrm{pre}} v.
\end{equation*}
Thus, every stationary point of the constrained maximization is a generalized eigenvector of the pair $(C_{\mathrm{task}}, C_{\mathrm{pre}})$. Moreover, left-multiplying by $v^\top$ and using the normalization $v^\top C_{\mathrm{pre}} v = 1$ yields
\begin{equation*}
\gamma = v^\top C_{\mathrm{task}} v = \rho(v),
\end{equation*}
so the generalized eigenvalue is equal to the achieved Rayleigh quotient ratio of utility-over-risk. Therefore, any direction inside the subspace spanned by the top generalized eigenvectors achieves the largest possible utility-to-risk ratio. Setting $v = C_\mathrm{pre}^{-1/2}r$ converts generalized eigenproblem to the ordinary eigenproblem $C_{\mathrm{pre}}^{-1/2} C_{\mathrm{task}} C_{\mathrm{pre}}^{-1/2} r
=
\gamma r$, or 
\begin{equation*}
\label{eq:eigenproblem}
C_{\mathrm{pre}}^{-1/2} C_{\mathrm{task}} C_{\mathrm{pre}}^{-1/2} 
=
 R \Gamma R^\top 
\end{equation*}
and solving it gives $\Gamma=\mathrm{diag}(\gamma_1,\ldots,\gamma_{d_\mathrm{in}})$ and $R = [r_1, \dots, r_{d_\mathrm{in}}]$. Therefore, corresponding generalized eigenvectors are
\begin{equation*}
V = C_{\mathrm{pre}}^{-1/2} R = [v_1,\ldots,v_{d_\mathrm{in}}], \quad V^\top C_{\mathrm{pre}} V = I.
\end{equation*}
Here, $\gamma_i$ is the adaptation utility per unit forgetting risk achieved by direction $v_i$. Large $\gamma_i$ values indicate directions that lead to high target adaptation utility while incurring a low forgetting risk, whereas small $\gamma_i$ values indicate directions with limited adaptation utility relative to the risk they incur. In the resulting generalized eigenspace coordinates, the forgetting risk is whitened, and the adaptation utility is diagonalized, i.e., $V^\top C_{\mathrm{pre}} V=I$ ($C_{\mathrm{pre}}$-orthogonal) and $V^\top C_{\mathrm{task}}V=\Gamma$. The adaptation-forgetting trade-off therefore decomposes into scalar direction scores indexed by the generalized eigenvalues. Fig.~\ref{fig:first_fig}(c)-(e) illustrate the whitened forgetting risk and diagonalized adaptation utility in the resulting generalized eigenspace coordinate system.


\subsection{Spectral-Gated AdamW (SAdam)} 
\label{sec:3.3}
Using the generalized eigenbasis from Sec.~\ref{sec:3.2}, we track AdamW's moment statistics~\cite{adamw} in the corresponding transformed generalized eigenspace coordinate, where the forgetting risk is normalized, and the adaptation-forgetting trade-off is comparable direction by direction. The weight matrix $W$ is expressed in the generalized eigenspace of $V$ through the following change of coordinates:
\begin{equation*}
    \bar W = W(V^\top)^{-1}=W C_{\mathrm{pre}}^{1/2}R \Leftrightarrow W = \bar W R^\top C_{\mathrm{pre}}^{-1/2}
\end{equation*}
Let $G_t = \nabla_W \mathcal{L}(W_t)$ be the gradient of the training objective with respect to the $W$ in the original Euclidean coordinates. By the definition, 
\begin{equation}
    d\mathcal L
=
\langle G_t, dW\rangle_F
=
\langle G_t, d\bar W\,R^\top C_{\mathrm{pre}}^{-1/2}\rangle_F
=
\langle G_t C_{\mathrm{pre}}^{-1/2}R, d\bar W\rangle_F.
\end{equation}
Hence, the gradient in the generalized eigenspace coordinate becomes
\begin{equation*}
\bar G_t=\nabla_{\bar W}\mathcal L = G_t C_{\mathrm{pre}}^{-1/2}R.
\end{equation*}
We then compute AdamW moment estimates in these transformed coordinates,
\begin{equation*}
    \bar m_t , \bar v_t = \operatorname{AdamWMoments}(\bar m_{t-1} , \bar v_{t-1}; \bar G_t).
\end{equation*}

The generalized eigenvalues $\{\gamma_i\}$ provide scalar scores for modulating each transformed direction. We therefore define the monotone gate function
\begin{equation*}
s_i = \psi(\gamma_i) = \frac{\gamma_i}{\gamma_i + \lambda} \in [0,1),
\qquad \lambda > 0.
\end{equation*}
This gate assigns larger weights to directions with high utility-to-risk ratios and smaller weights to directions with low ratios. We apply the gate to the preconditioned AdamW step.
Let $\tilde{s}_i=\max\{s_i,s_{\min}\}$ and $\tilde{s}=(\tilde{s}_1,\ldots,\tilde{s}_{d_{\mathrm{in}}})$. The weight update in the generalized eigenspace is
\begin{equation*}
\Delta \bar{W}_t
=
-\eta
\left(
\frac{{\bar m}_t}{\sqrt{{\bar v}_t}+\epsilon}
\right)
\operatorname{diag}(\tilde{s}).
\end{equation*}
Thus, weight update along each coordinate receives its own learning rate. Finally, we map the update from the generalized eigenspace back to the original Euclidean parameterization:
\begin{equation*}
\Delta W_t = \Delta \bar{W}_t R^\top C_\mathrm{pre}^{-1/2},
\qquad
W_{t+1} = W_t + \Delta W_t.
\end{equation*}
The gating profile is controlled by two hyperparameters, $\lambda$ and $s_{\mathrm{min}}$: $s_{\mathrm{min}}$ sets the minimum gating values, while $\lambda$ controls how sharply the gate separates high-ratio and low-ratio directions. A smaller $s_{\mathrm{min}}$ or larger $\lambda$ makes the update more conservative by more strongly attenuating low-ratio directions, whereas a larger $s_{\mathrm{min}}$ or smaller $\lambda$ allows a more exploratory adaptation strategy. 

\subsection{Model-Generated Pretraining-Proxy Calibration}
\label{sec:selfseeded}
We next describe how OGE constructs the pretraining-proxy calibration dataset used to estimate $C_{\mathrm{pre}}$. Existing methods such as CorDA~\cite{corda} and LoRA-Null~\cite{loranull} use a fixed external dataset, such as NQ Open~\cite{nqopen}, as a proxy for the pretraining corpus $D_{\mathrm{pre}}$. However, a single dataset provides a narrow and biased view of the diverse distribution encountered during pretraining, resulting in limited coverage of foundation knowledge and consequently weaker protection for behaviors not well represented in that dataset. In contrast, OGE approximates $D_{\mathrm{pre}}$ by sampling sequences from the autoregressive distribution of the pretrained model. This yields a model-derived proxy that is not tied to any single benchmark, enabling a more robust estimation of pretraining data distribution and, consequently, providing broader and stronger protection.

\begin{theorem}
\label{thm:kl-control}
Let \(\mathcal V\) be a finite vocabulary and let
\(P_{\mathrm{data}}\in\Delta(\mathcal V^T)\) denote the pretrain corpus distribution
over token sequences \(X=(x_1,\ldots,x_T) \in \mathcal{V}^T\), where \(T<\infty\). We denote the autoregressive generative distribution induced by a pretrained model with parameters $\theta$ as
$
    P_\theta(x_{1:T})
    =
    \prod_{t=1}^T P_{\theta}(x_t\mid x_{<t}) .   
$
For any fixed \(\alpha\in(0,1)\) and \(\eta>0\), there exists a parameter choice \(\theta\)
such that
\begin{equation*}
    D_{\mathrm{KL}}
    \left(
        P_{\mathrm{data}}
        \,\middle\|\,
        P_\theta
    \right)
    \leq
    T
    \left(
        -\log(1-\alpha)
        +
        2\eta
    \right).
    \label{eq:kl-bound-alpha-eta}
\end{equation*}
Consequently, if \(\alpha\) and \(\eta\) are such that
\(T(-\log(1-\alpha)+2\eta)\leq \epsilon_{\mathrm{KL}}\), then
\(D_{\mathrm{KL}}(P_{\mathrm{data}}\|P_\theta)\leq \epsilon_{\mathrm{KL}}\).
\end{theorem}

The uniform approximation premise in Theorem~\ref{thm:kl-control} is motivated by the universal approximation property of Transformer architectures~\cite{yu2026attentionheadcount, uat, jiang2024approximation}. In particular, autoregressive distributions of sufficiently expressive Transformer classes can approximate the pretraining distribution. The proof of Theorem~\ref{thm:kl-control} is provided in Appendix~\ref{app:theorem-proofs}.


Motivated by Theorem~\ref{thm:kl-control}, we approximate pretraining data $D_{\mathrm{pre}}$ through model-generated samples. Given prompts $\rho_1,\ldots,\rho_N$, we independently sample an output sequence from the model conditioned on each prompt,
\begin{equation*}
    c_i \sim p_{\theta_0}(x_{1:\ell} \mid \rho_i),
    \qquad i=1,\ldots,N,
\end{equation*}
where $\theta_0$ denotes the pretrained model and $\ell$ is the generation length. Unless otherwise specified, we set $\rho_i$ to the $\mathrm{BOS}$ token for unconditional generation. This yields the model-generated calibration set
$D_{\mathrm{MG}} \triangleq \{c_i\}_{i=1}^N$, which serves as a proxy of pretraining dataset $D_{\mathrm{pre}}$. We estimate $C_{\mathrm{pre}}$ by running a forward pass on $D_{\mathrm{MG}}$ and collecting the corresponding layer-input token sequence.

\section{Experiments}
\label{exp}

\paragraph{Models and Datasets.}
Following prior work~\cite{loranull, corda, milora}, we perform supervised finetuning of Llama2-7B~\cite{llama2} on three downstream adaptation tasks: math, code, and instruction following. To evaluate preservation of pretrained capabilities, we use as benchmarks TriviaQA~\cite{triviaqa}, NQ Open~\cite{nqopen}, and WebQS~\cite{webqs}. For math adaptation, models are trained on MetaMathQA~\cite{metamath} and evaluated on GSM8K~\cite{gsm8k} and MATH~\cite{math}. For code adaptation, models are trained on CodeFeedback~\cite{codefeedback} and evaluated on HumanEval~\cite{humaneval} and MBPP~\cite{mbpp}. For instruction following, models are trained on WizardLM-Evol-Instruct~\cite{wizard} and evaluated on IFEVAL~\cite{ifeval}. We report exact-match scores on the preservation benchmarks; we also report Avg1, the arithmetic mean of the three preservation scores, and Avg1(\%), the preservation ratio relative to the original pretrained model. For downstream adaptation, we report task-specific scores, as well as the average of those scores, Avg2. To jointly evaluate task adaptation and preservation, we report GM, the geometric mean of Avg1 and Avg2.


\begin{table*}[t]
\centering
\small

\begin{subtable}{\linewidth}
\centering
\subcaption{LLaMA-2-7B finetuned on math (MetaMathQA)}
\label{tab:llama2_math}
\setlength{\tabcolsep}{1mm}
\resizebox{\textwidth}{!}{%
\begin{tabular}{l|c|ccccc|ccc|c}
    \toprule[1.5pt]
    & & \multicolumn{5}{c|}{\textbf{Foundation-Preservation}} & \multicolumn{3}{c|}{\textbf{Adaptation}} & \\
    \cmidrule(lr){3-7} \cmidrule(lr){8-10}
    Method & \textbf{\#Param} & Trivia QA & NQ Open & WebQS & Avg1 & Avg1(\%) & GSM8K & Math & Avg2 & GM \\
    \midrule
    LLaMA-2-7B & - & 52.51 & 16.91 & 5.88 & 25.10 & 100.00 & - & - & - & - \\
    \midrule
    LoRA~\cite{lora}      & 320M & 45.06      & 1.535       & 6.64        & 17.74      & 70.69 & 42.92      & 6.09        & 24.50     & 20.85 \\
    MiLoRA~\cite{milora}    & 320M & 48.28      & 3.91 & 6.30         & 19.5        & 77.68 & 40.51       & 5.44        & 22.97      & 21.16 \\
    CorDA~\cite{corda}     & 320M & 47.77      & 5.91        & 6.49 & 20.05 & 79.91 & 40.26       & 5.65 & 22.95 & 21.45 \\
    LoRA-Null~\cite{loranull} & 320M & \textbf{49.93} & \underline{7.14} & 6.58 & \underline{21.22} & \underline{84.55} & \underline{43.89} & \underline{6.36} & \underline{25.12} & \underline{23.08} \\
    OPLoRA~\cite{oplora}    & 320M & 49.14      & 2.67        & \underline{7.48}        & 19.76 & 78.73 & 31.53       & 4.32        & 17.92      & 18.82 \\
    \midrule
    \rowcolor{metablue!10}
    FoLoRA (ours) & 320M & \underline{49.16} & \textbf{17.09} & \textbf{11.95} & \textbf{26.06} & \textbf{103.84} & \textbf{46.09} & \textbf{6.60} & \textbf{26.34} & \textbf{26.20} \\
    \bottomrule[1.5pt]
\end{tabular}
}
\end{subtable}

\vspace{2mm}

\begin{subtable}{\linewidth}
\centering
\subcaption{LLaMA-2-7B finetuned on code (CodeFeedback)}
\label{tab:llama2_code}
\setlength{\tabcolsep}{1mm}
\resizebox{\textwidth}{!}{%
\begin{tabular}{l|c|ccccc|ccc|c}
    \toprule[1.5pt]
    Method & \textbf{\#Param} & Trivia QA & NQ Open & WebQS & Avg1 & Avg1(\%) & HumanEval & MBPP & Avg2 & GM \\
    \midrule
    LLaMA-2-7B & - & 52.51 & 16.91 & 5.88 & 25.10 & 100.00 & - & - & - & - \\
    \midrule
    LoRA~\cite{lora}      & 320M & \textbf{51.36} & 10.77 & \underline{8.56} & \underline{23.57} & \underline{93.87} & 16.99 & 21.40 & 19.19 & 21.26 \\
    MiLoRA~\cite{milora}    & 320M & 49.28 & 11.80 & 7.10 & 22.72 & 90.54 & 17.49 & 20.22 & 18.86 & 20.70 \\
    CorDA~\cite{corda}     & 320M & 49.32 & 12.72 & 6.72 & 22.92 & 91.33 & 17.45 & 20.48 & 18.97 & 20.85 \\
    LoRA-Null~\cite{loranull} & 320M & 49.09 & \underline{13.78} & 6.61 & 23.16 & 92.27 & \underline{17.63} & 24.04 & \underline{20.83} & \underline{21.97} \\
    OPLoRA~\cite{oplora}  & 320M & 44.32 & 9.07 & 6.10 & 19.83 & 79.00 & 15.24 & \underline{25.20} & 20.22 & 20.02\\
    \midrule
    \rowcolor{metablue!10}
    FoLoRA (ours) & 320M & \underline{51.31} & \textbf{14.78} & \textbf{10.82} & \textbf{25.63} & \textbf{102.12} & \textbf{17.68} & \textbf{26.80} & \textbf{22.24} & \textbf{23.87} \\
    \bottomrule[1.5pt]
\end{tabular}
}
\end{subtable}

\begin{subtable}{\linewidth}
\centering
\subcaption{LLaMA-2-7B finetuned on instruction-following (WizardLM-Evol-Instruct)}
\label{tab:llama2_inst}
\setlength{\tabcolsep}{1mm}
\begin{tabular}{l|c|ccccc|c|c}
    \toprule[1.5pt]
    Method & \textbf{\#Param} & Trivia QA & NQ Open & WebQS & Avg1 & Avg1(\%) & IFEVAL & GM \\
    \midrule
    LLaMA-2-7B    & -    & 52.51 & 16.91 & 5.88 & 25.10      & 100.00          &   -    &   -   \\
    \midrule
    LoRA~\cite{lora}          & 320M & 42.76 & 7.72  & 6.34 & 18.94      & 75.45  & 32.49 &  24.80    \\
    MiLoRA~\cite{milora}        & 320M & 45.98 & 11.94 & 6.74 & 21.55 & 85.86 &   32.21    &    26.34  \\
    CorDA~\cite{corda}         & 320M & 45.34 & \underline{14.50}  & \underline{7.28} & 22.37 & 89.13 &   \underline{34.45}    &    \underline{27.76}  \\
    LoRA-Null~\cite{loranull}     & 320M & \underline{47.63} & \textbf{14.68} & 6.88 & \underline{23.06} & \underline{91.87} & 32.61 &   27.42   \\
    OPLoRA~\cite{oplora}        & 320M & 46.90  & 10.19 & 7.08 & 21.39      & 85.21  & 28.65 &   24.75   \\
    \midrule
    \rowcolor{metablue!10}
    FoLoRA (ours) & 320M & \textbf{49.14} & 14.01 & \textbf{7.97} & \textbf{23.71} & \textbf{94.44} & \textbf{36.13} &  \textbf{29.26}    \\
    \bottomrule[1.5pt]
\end{tabular}
\end{subtable}
\caption{
Main results on \colorbox{metablue!10}{\textbf{LLaMA2-7B-Base}}. Avg1 averages preservation scores over TriviaQA, NQ Open, and WebQS; Avg1(\%) reports its percentage relative to the pretrained model. 
Avg2 averages downstream scores, while IFEVAL is used directly for instruction following. 
GM is the geometric mean of Avg1 and Avg2. 
Bold indicates the best result and underlining the runner-up. 
Results are averaged over three random seeds.}
\label{tab:llama2}

\end{table*}

\paragraph{Baselines.}
We compare OGE with several forgetting-aware fine-tuning methods, including LoRA~\cite{lora}, MiLoRA~\cite{milora}, CorDA~\cite{corda}, LoRA-Null~\cite{loranull}, and OPLoRA~\cite{oplora}. We standardize the training and evaluation pipeline across methods, while preserving each baseline’s method-specific procedure as described in its original work. We also report performance of the original Llama2-7B~\cite{llama2} for reference. Because these baselines rely on LoRA parameterization, we instantiate OGE under the same parameterization for a direct comparison. In this OGE+LoRA setting, we train LoRA adapters using the SAdam optimizer. Further implementation details are provided in Appendix~\ref{appen:experiments}.

\subsection{Main Results}
\label{sec:main_result}
Across the three fine-tuning settings on Llama2-7B, OGE achieves the best overall preservation-adaptation balance among the evaluated methods, as shown in Tables~\ref{tab:llama2_math}, \ref{tab:llama2_code} and \ref{tab:llama2_inst}. On the math adaptation task, OGE obtains the highest preservation scores, with Avg1 = 26.06 and Avg1(\%) = 103.84, while also achieving the best downstream adaptation, with Avg2 = 26.34 and GM = 26.20. On the code adaptation task, OGE again yields the strongest preservation, with Avg1 = 25.63 and Avg1(\%) = 102.12, and achieves the best aggregate performance, with Avg2 = 22.24 and GM = 23.87. On the instruction following task, OGE obtains the highest preservation score, with Avg1 = 23.71 and Avg1(\%) = 94.44, and the best combined score, GM = 29.26. Overall, OGE achieves the strongest preservation and downstream performance, demonstrating that it improves adaptation without compromising pretrained knowledge.

\paragraph{Ablation on Spectral-gated Adam.}
\begin{figure*}[t]
  \centering
  \makebox[\textwidth][c]{%
    \begin{subfigure}[t]{0.37\textwidth}
      \centering
      \includegraphics[width=\linewidth]{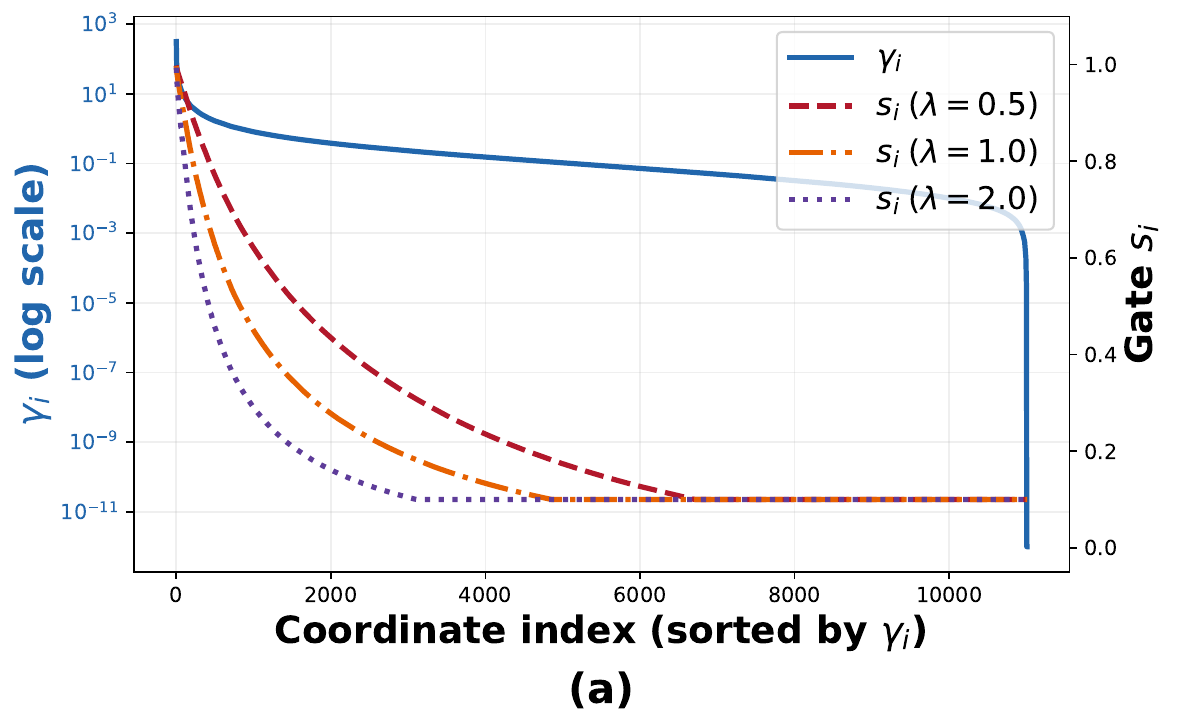}
    \end{subfigure}%
    \begin{subfigure}[t]{0.4\textwidth}
      \centering
      \includegraphics[width=\linewidth]{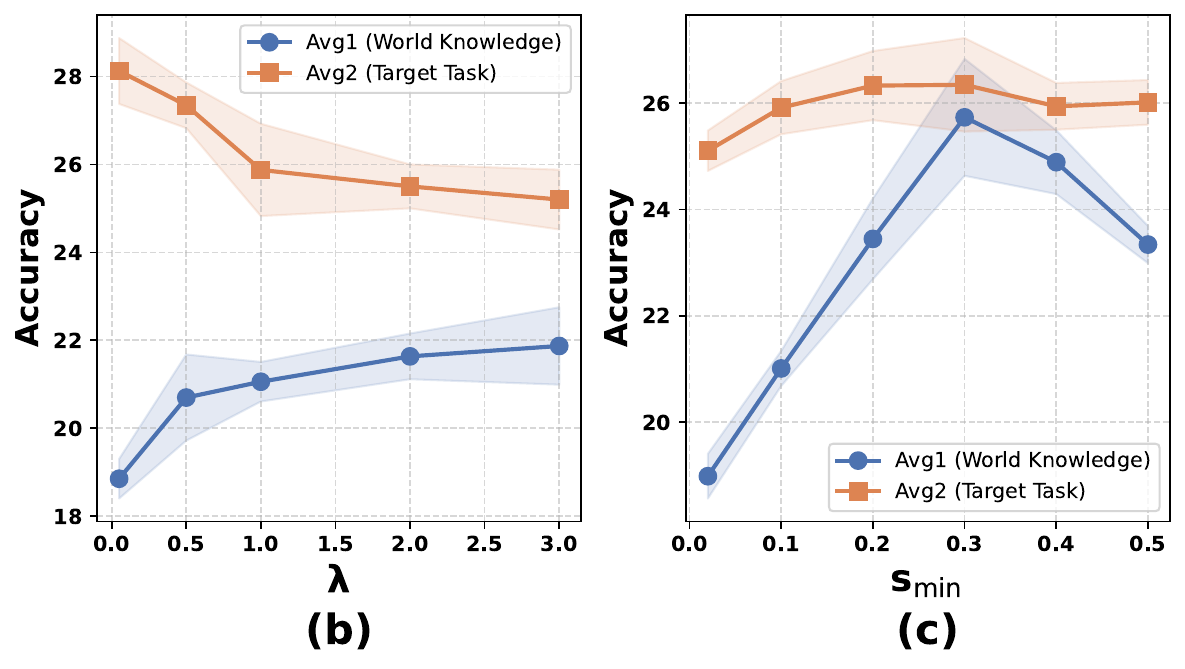}
    \end{subfigure}
    \begin{subfigure}[t]{0.22\textwidth}
      \centering
      \includegraphics[width=\linewidth]{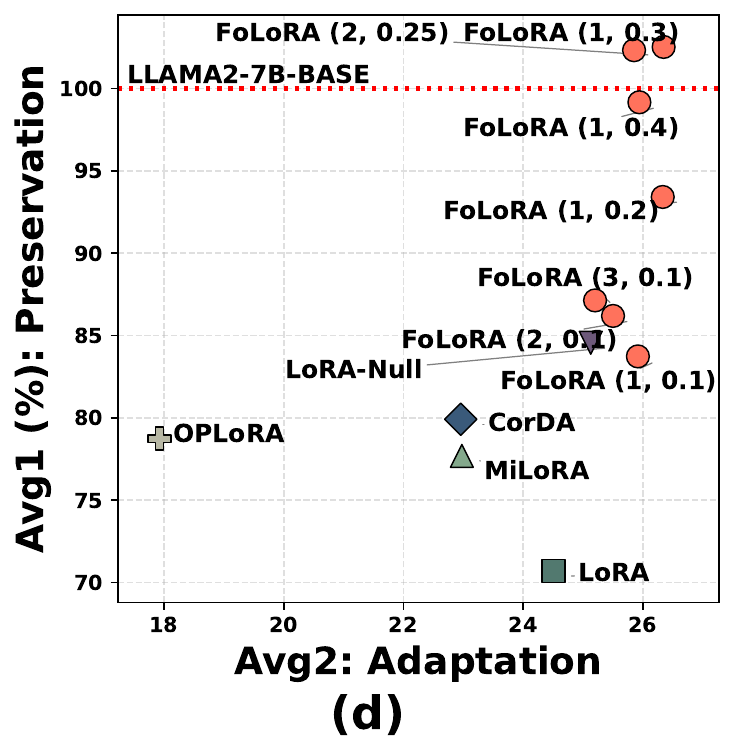}
    \end{subfigure}

  }
  \caption{Ablation of spectral-gated AdamW (SAdam) on Llama2-7B fine-tuned on MetaMathQA.
    (a) Distribution of generalized eigenvalues $\{\gamma_i\}$ and corresponding gate values $\{s_i\}$ for different $\lambda$ (Layer-15 MLP downproj, $s_\mathrm{min}=0.1$).
    (b) Preservation score Avg1 and downstream adaptation score Avg2 as $\lambda$ varies.
    (c) Avg1 and Avg2 as $s_{\min}$ varies.
    (d) Preservation-adaptation trade-off of OGE across operating points $(\lambda,s_{\min})$, compared with baselines.
}
  \label{fig:OGE_ablation}
\end{figure*}

\begin{table*}[t]
\centering
\small


\resizebox{\textwidth}{!}{%
\begin{tabular}{l|cccccccccc|c|c}
    \toprule[1.5pt]

    Method & IFEVAL & TruthfulQA & SQuAD & TriviaQA & NQopen & MMLU & BLong & LBench & GPQA & Avg1 & Math & GM \\
    \midrule
    Qwen3-1.7B & 20.51 & 48.76 & 27.19 & 28.78 & 8.81 & 60.61 & 64.54 & 29.42 & 31.31 & 35.54 & - & - \\
    \midrule
    LoRA      & 19.59 & 47.93 & 25.79 & 29.62 &6.06 & \textbf{60.12} & 65.74 & 29.22 & 27.77 & 34.64 & 42.74 & 38.48 \\
    MiLoRA      & 23.10 & 48.06 & 23.07 & 29.69 &5.40 & 59.90 & 65.98 & 28.03 & 29.79 & 34.78 & 43.98 & 39.11 \\
    OPLoRA  & 19.77 & 48.06 & 25.52 & 29.59 &\textbf{7.97} & 59.94 &  65.23 & 29.02 & 29.29 & 34.93 & 43.02 & 38.76   \\
    \midrule
    CorDA       & 18.66 & 48.30 & 23.93 & 16.58 &5.73& 56.84 & 65.57 & 27.83 & 29.79 & 32.58 & 42.64 & 37.27 \\
    LoRA-Null  & 23.29 & \underline{48.40} & 23.53 & 30.53 &5.37 & 59.96 & 66.07 & 28.03 & 27.77 & 34.77 & 43.86 & 39.05 \\
    CorDA (4B)      &19.03 &48.35&23.76&19.85&7.50&57.11&66.06&27.83&27.27&32.97&42.74&37.54 \\
    LoRA-Null (4B)  &21.96 &48.16&25.84&30.58&7.45&60.07&65.52&\underline{29.42}&28.78 &35.30&43.12&39.01 \\
    \midrule
    \rowcolor{gray!8}
    FoLoRA (NQ) &\underline{23.29} &48.32&\underline{27.56}&\underline{32.41}&\underline{7.73}&\underline{60.11}&\underline{66.34} &29.02&\underline{31.81}&\underline{36.28}&\textbf{44.28}&\underline{40.08} \\
    \rowcolor{gray!8}
    FoLoRA (0.6B) &23.10& 48.22& 25.53& 32.18& 6.34& 59.86& 66.09& 28.03& 31.31& 35.62& 43.62& 39.42\\
    \rowcolor{gray!8}
    FoLoRA (1.7B) & 22.50& 48.36& \textbf{27.98}& \textbf{32.46} &7.00 & 59.97& \textbf{66.50}& 28.82& 30.30& 35.98& 43.52& 39.57\\
    \rowcolor{metablue!10}
    FoLoRA (4B,ours) & \textbf{23.84} & \textbf{48.46} & 27.07 & 32.07 &7.25 & 60.06 & 66.17 & \textbf{29.62} & \textbf{34.34} & \textbf{36.54} & \underline{44.02} & \textbf{40.10} \\
    
    \bottomrule[1.5pt]
\end{tabular}
}
\caption{Results on \colorbox{metablue!10}{\textbf{Qwen3-1.7B-Base}} after fine-tuning on MetaMathQA, evaluated on preservation benchmarks and a downstream math task. We also report CorDA and LoRA-Null using the model-generated pretraining-proxy calibration set, and FoLoRA calibrated on NQ Open. Bold denotes the best result, the runner-up is underlined.}
\label{tab:qwen_math2}
\end{table*}

Fig.~\ref{fig:OGE_ablation}(a) visualizes the distribution of generalized eigenvalues $\gamma_i$ and the corresponding gate values $s_i$ under different choices of $\lambda$. We observe that a small number of leading eigenvalues account for most of the spectral mass, indicating large variation in utility-to-risk ratios across the generalized eigenvectors. Directions with high adaptation utility per unit forgetting risk have large Rayleigh quotients and therefore large $\gamma_i$, for which the gating values $s_i$ approach $1$, allowing nearly full AdamW updates along these directions. In contrast, directions with low utility-to-risk ratios are characterized by small $\gamma_i$ and assigned smaller gating values, which effectively suppresses updates that are less favorable for preservation-aware adaptation. 

Fig.~\ref{fig:OGE_ablation}(b) reports Avg1, which measures foundation preservation, and Avg2, which measures downstream-task performance, as a function of $\lambda$. As $\lambda$ increases, Avg1 rises consistently, whereas Avg2 decreases gradually. This trend indicates that larger $\lambda$ makes the update rule more conservative by focusing optimization on directions with the largest utility-to-risk ratios, thus improving preservation at the cost of reduced adaptation capacity. Fig.~\ref{fig:OGE_ablation}(c) shows the effect of varying $s_{\mathrm{min}}$. As $s_{\mathrm{min}}$ increases, Avg2 modestly improves, since a larger minimum gate allows more update magnitude along directions that would otherwise be strongly suppressed. Avg1, however, follows a non-monotonic pattern. A small to moderate value of $s_{\mathrm{min}}$ improves preservation performance, suggesting positive backward transfer~\citet{bt}. This suggests that overly aggressive suppression of low $\gamma_i$ directions may restrict optimization too strongly. However, once $s_{\mathrm{min}}$ becomes too permissive, aggressive updates along directions with larger forgetting risk start to degrade pretrained knowledge. Fig.~\ref{fig:OGE_ablation}(d) jointly compares the preservation–adaptation performance of OGE across operating points induced by different ($\lambda,s_{\min}$) settings against the baseline methods. OGE traces a favorable preservation–adaptation frontier because optimization is carried out in the generalized eigenspace, where this trade-off is optimized.

\begin{figure*}[t]
  \centering
  \makebox[\textwidth][c]{%
    \begin{subfigure}[t]{0.29\textwidth}
      \centering
      \includegraphics[width=\linewidth]{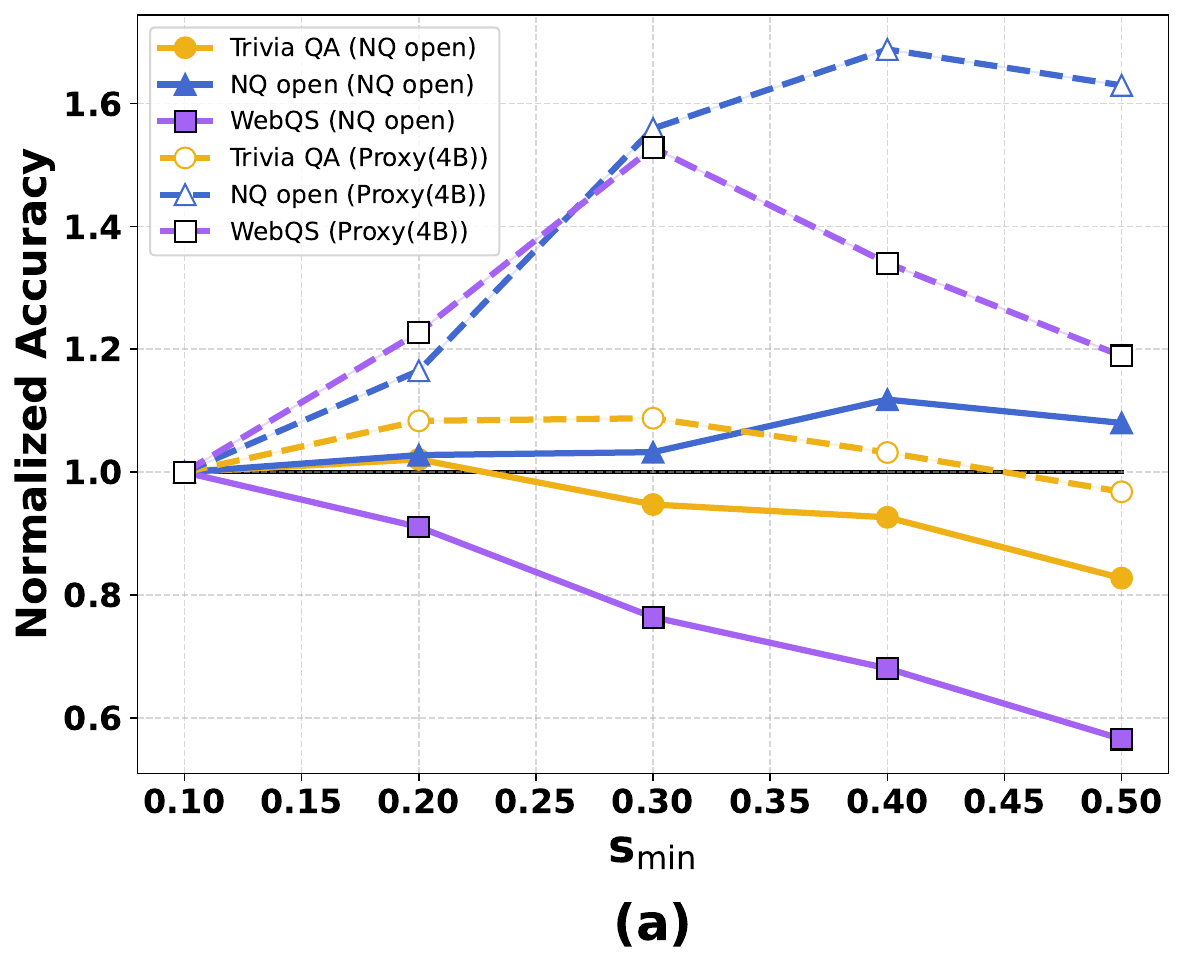}
    \end{subfigure}
    \begin{subfigure}[t]{0.3\textwidth}
      \centering
      \includegraphics[width=\linewidth]{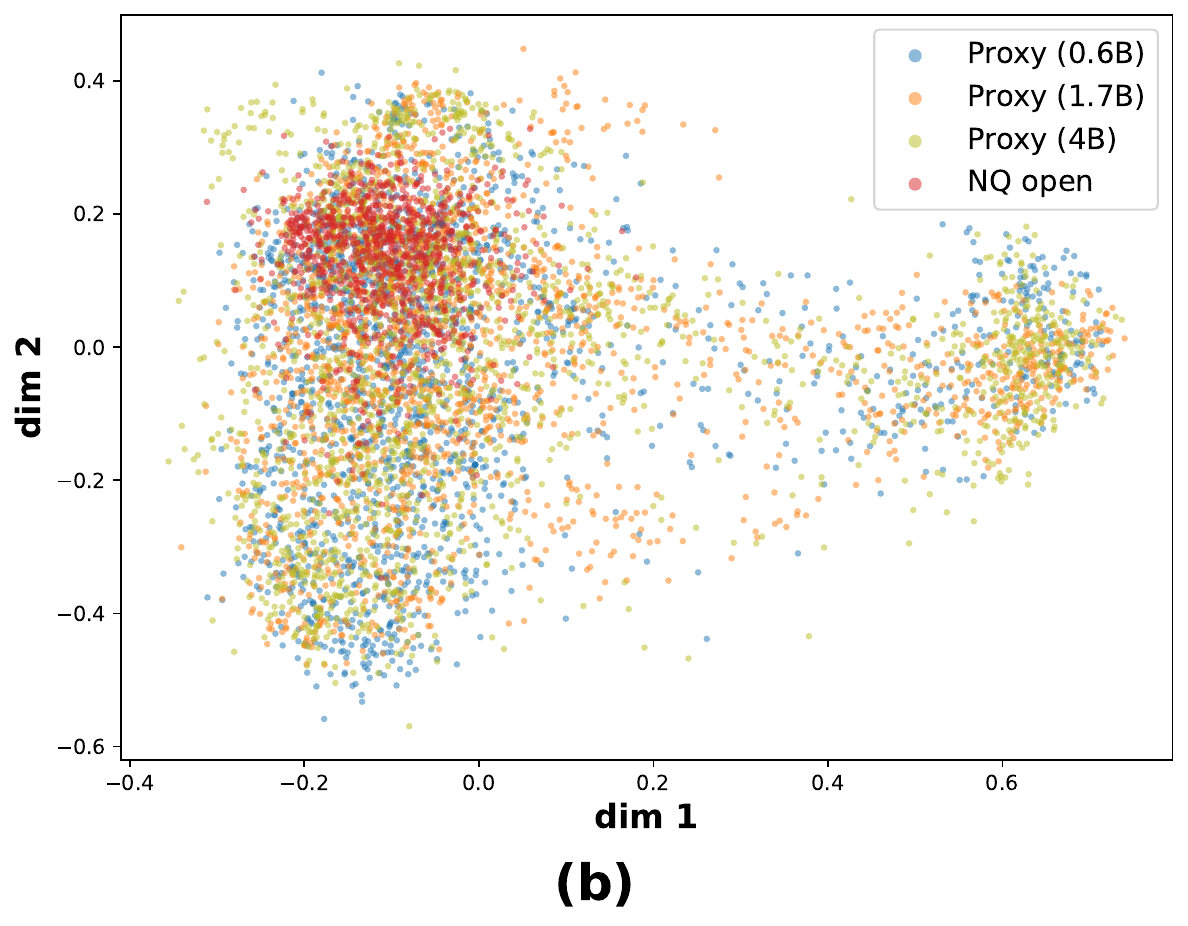}
    \end{subfigure}%
    \begin{subfigure}[t]{0.4\textwidth}
      \centering
      \includegraphics[width=\linewidth]{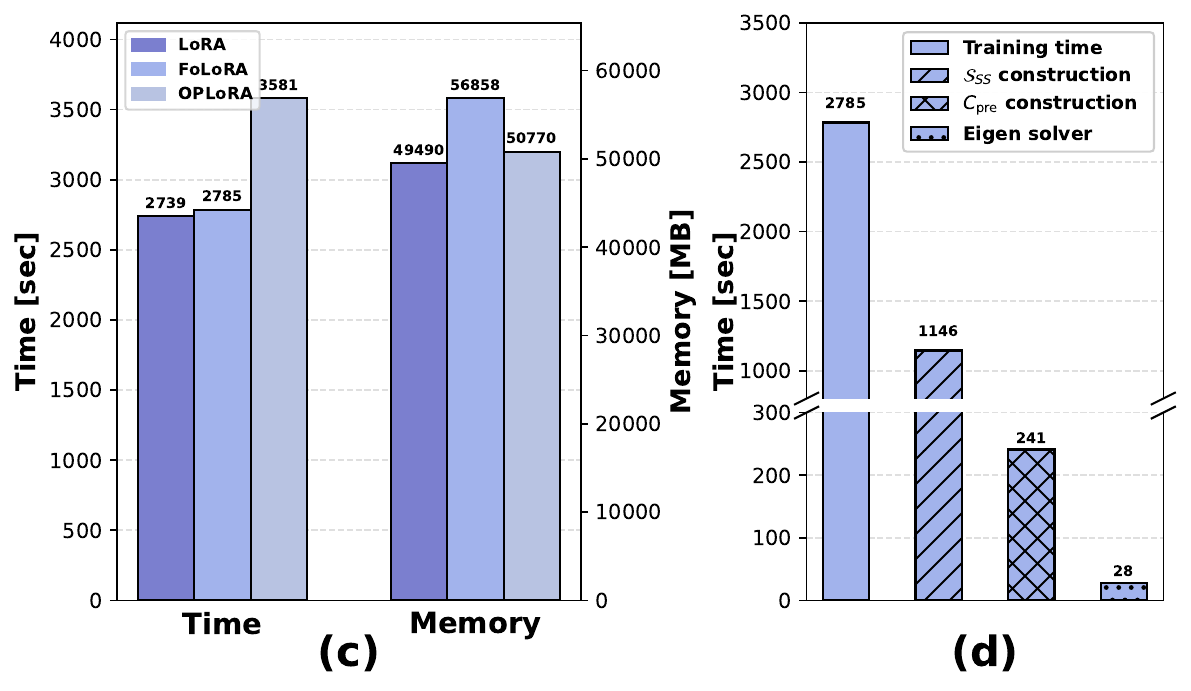}
    \end{subfigure}%
  }
\caption{
(a) Normalized accuracy on preservation benchmarks as a function of $s_{\min}$ (LLaMA-2-7B, math adaptation);
(b) PCA visualization of NQ Open queries and model-generated calibration sets produced by Qwen3-0.6B, 1.7B, and 4B; (c) training time and peak VRAM usage (Qwen3-1.7B, math adaptation);
(d) non-training overhead components of FoLoRA, including model-generated calibration construction, $C_{\mathrm{pre}}$ construction, and generalized eigensolver cost (Qwen3-1.7B, math adaptation; Qwen3-4B used for calibration generation).
}
 
\label{fig:lastfig}
\end{figure*}

\paragraph{Evaluation on Broader Preservation Benchmarks.} We further evaluate OGE on a broader set of preservation benchmarks, as shown in Table~\ref{tab:qwen_math2}. Specifically, we consider instruction following (IFEVAL; \citealt{ifeval}), hallucination robustness (TruthfulQA; \citealt{truthfulqa}), reading comprehension (SQuAD-v2; \citealt{squad}), multitask knowledge and reasoning (MMLU; \citealt{mmlu}), long-context modeling (BABILong; \citealt{babilong}, LongBench-v2; \citealt{longbenchv2}), and scientific reasoning (GPQA; \citealt{gpqa}). All experiments are conducted on Qwen3-1.7B with MetaMathQA fine-tuning~\cite{metamath}. When comparing the base model with its LoRA-fine-tuned counterpart, we observe more severe degradation on tasks that require factual retrieval, particularly NQ Open, GPQA, and MMLU, consistent with prior findings~\cite{mapping}. Instruction-following and hallucination robustness degrade only mildly, as reflected by IFEVAL and TruthfulQA; Interestingly, long-context performance, evaluated on LongBench-v2 and BABILong, remains largely preserved. These results indicate that fine-tuning does not affect all non-target capabilities uniformly, making broad preservation evaluation important. OGE, with the default model-generated calibration setting denoted OGE (4B), achieves the strongest overall preservation across these benchmarks while maintaining competitive downstream adaptation.

\paragraph{Overhead Analysis.} FoLoRA adds modest overhead: in the Qwen3-1.7B math-adaptation setting, fine-tuning time increases by 1.6\% and peak VRAM by 14.8\% over LoRA which serves as a lower-bound (Fig.~\ref{fig:lastfig}(c)). In addition, the one-time preprocessing costs for calibration generation, covariance construction, and generalized eigendecomposition remain lightweight, as shown in Fig.~\ref{fig:lastfig}(d).

\subsection{Effectiveness of Model-generated Calibration Set}
In Table~\ref{tab:qwen_math2}, we evaluate the effect of the calibration source by reporting performance of CorDA and LoRA-Null with our model-generated calibration set $D_{\mathrm{MG}}$, and OGE with NQ Open as the calibration set. We construct $D_{\mathrm{MG}}$ using generators of varying sizes; e.g., OGE (4B) and CorDA (4B) use Qwen3-4B generations. The results show that replacing NQ Open with $D_{\mathrm{MG}}$ improves the preservation performance of CorDA and LoRA-Null, suggesting that model-generated calibration provides activation statistics that cover a broader range of pretrained behaviors than a single dataset. Conversely, when OGE uses NQ Open as the calibration set, its preservation effect is concentrated on factual-retrieval benchmarks such as TriviaQA and NQ Open, while performance on several other domains is lower than with $D_{\mathrm{MG}}$. These results show that the calibration source matters: model-generated calibration improves preservation across domains without sacrificing adaptation.

Using larger Qwen3 models to generate $D_{\mathrm{MG}}$ generally improves OGE’s overall preservation performance. Since these model series are trained on the same pretrain corpus~\cite{qwen_report}, the improvement suggests that larger models generate calibration data whose distribution more closely approximates the pretrain corpus distribution. The trend is consistent with the approximation-rate discussion in Appendix~\ref{appen:kl-approx-rate}, where increased model cardinality reduces the approximation gap under suitable assumptions.

To further analyze calibration coverage, we vary OGE’s minimum spectral gate $s_{\mathrm{min}}$, which controls forgetting pressure, and report normalized accuracy (relative to $s_{\mathrm{min}}=0.1)$ on TriviaQA, NQ Open, and WebQS in Fig.~\ref{fig:lastfig}(a). As $s_{\mathrm{min}}$ increases, OGE calibrated on NQ Open primarily preserves performance on NQ Open, while other benchmarks degrade. In contrast, $D_{\mathrm{MG}}$ calibration preserves performance more consistently across all three benchmarks, even under stronger forgetting pressure. These results indicate that a single calibration benchmark can bias preservation, whereas model-generated calibration provides broader protection coverage. Figure~\ref{fig:lastfig}(b) visualizes the distributional coverage of NQ Open and the model-generated calibration sets generated from Qwen3 0.6B, 1.7B, and 4B using principal component analysis (PCA)~\cite{pca}. The model-generated proxies exhibit broader coverage and encompass much of the distributional support of NQ Open.

\section{Conclusion, Limitations, and Future Directions}

We propose OGE, a forgetting-aware optimization framework. OGE uses a generalized Rayleigh quotient to score update directions by task utility per unit forgetting penalty, estimated using a model-generated calibration set, and then applies direction-wise gated AdamW updates in that generalized eigenspace. Experiments show improved preservation–adaptation trade-offs. 

Although instantiated with LoRA~\cite{lora} for fair comparison with baselines, OGE is architecture-agnostic and can be extended to full-weight fine-tuning. Scaling OGE to modern large-scale models and comprehensively evaluating this extension remain valuable directions for future work. A remaining limitation is the additional VRAM cost during training. While modest, investigating the memory efficiency of OGE would be valuable. Finally, our current study intentionally focuses on supervised fine-tuning because of its widespread use, practical effectiveness, and particular susceptibility to catastrophic forgetting. Extending OGE to other post-training paradigms, such as reinforcement learning and on-policy distillation, would also be valuable.

\section{Acknowledgments}
The authors acknowledge the computing resources made available by the Vista GPU Cluster, operated through the Center for Generative AI (CGAI) and the Texas Advanced Computing Center (TACC) at the University of Texas at Austin, which supported this research.

\bibliography{CITE}

@String(ICLR = {Int. Conf. Learn. Represent.})

@String(AAAI = {AAAI})

@String(ICLR  = {ICLR})

@article{llm1,
  title={Language Models are Few-Shot Learners},
  author={Tom B. Brown and Benjamin Mann and Nick Ryder and Melanie Subbiah and Jared Kaplan and Prafulla Dhariwal and Arvind Neelakantan and Pranav Shyam and Girish Sastry and Amanda Askell and Sandhini Agarwal and Ariel Herbert-Voss and Gretchen Krueger and Thomas Henighan and Rewon Child and Aditya Ramesh and Daniel M. Ziegler and Jeff Wu and Clemens Winter and Christopher Hesse and Mark Chen and Eric Sigler and Ma-teusz Litwin and Scott Gray and Benjamin Chess and Jack Clark and Christopher Berner and Sam McCandlish and Alec Radford and Ilya Sutskever and Dario Amodei},
  journal={ArXiv},
  year={2020},
  volume={abs/2005.14165},
  url={https://api.semanticscholar.org/CorpusID:218971783}
}

@article{llm2,
  title={Training language models to follow instructions with human feedback},
  author={Long Ouyang and Jeff Wu and Xu Jiang and Diogo Almeida and Carroll L. Wainwright and Pamela Mishkin and Chong Zhang and Sandhini Agarwal and Katarina Slama and Alex Ray and John Schulman and Jacob Hilton and Fraser Kelton and Luke E. Miller and Maddie Simens and Amanda Askell and Peter Welinder and Paul Francis Christiano and Jan Leike and Ryan J. Lowe},
  journal={ArXiv},
  year={2022},
  volume={abs/2203.02155},
  url={https://api.semanticscholar.org/CorpusID:246426909}
}

@article{instfollowing,
  title={Fine-Tuning Language Models from Human Preferences},
  author={Daniel M. Ziegler and Nisan Stiennon and Jeff Wu and Tom B. Brown and Alec Radford and Dario Amodei and Paul Christiano and Geoffrey Irving},
  journal={ArXiv},
  year={2019},
  volume={abs/1909.08593},
  url={https://api.semanticscholar.org/CorpusID:202660943}
}

@article{forget2,
  title={Fine-Tuning is Fine, if Calibrated},
  author={Zheda Mai and Arpita Chowdhury and Ping Zhang and Cheng-Hao Tu and Hong-You Chen and Vardaan Pahuja and Tanya Y. Berger-Wolf and Song Gao and Charles V. Stewart and Yu Su and Wei-Lun Chao},
  journal={ArXiv},
  year={2024},
  volume={abs/2409.16223},
  url={https://api.semanticscholar.org/CorpusID:272831906}
}

@article{forget3,
  title={Fine-tuning Aligned Language Models Compromises Safety, Even When Users Do Not Intend To!},
  author={Xiangyu Qi and Yi Zeng and Tinghao Xie and Pin-Yu Chen and Ruoxi Jia and Prateek Mittal and Peter Henderson},
  journal={ArXiv},
  year={2023},
  volume={abs/2310.03693},
  url={https://api.semanticscholar.org/CorpusID:263671523}
}

@inproceedings{milora,
  title={MiLoRA: Harnessing Minor Singular Components for Parameter-Efficient LLM Finetuning},
  author={Hanqing Wang and Zeguan Xiao and Yixia Li and Shuo Wang and Guanhua Chen and Yun Chen},
  booktitle={North American Chapter of the Association for Computational Linguistics},
  year={2024},
  url={https://api.semanticscholar.org/CorpusID:270440848}
}

@article{corda,
  title={CorDA: Context-Oriented Decomposition Adaptation of Large Language Models},
  author={Yibo Yang and Xiaojie Li and Zhongzhu Zhou and Shuaiwen Leon Song and Jianlong Wu and Liqiang Nie and Bernard Ghanem},
  journal={ArXiv},
  year={2024},
  volume={abs/2406.05223},
  url={https://api.semanticscholar.org/CorpusID:270370784}
}

@inproceedings{loranull,
  title={Put the space of lora initialization to the extreme to preserve pre-trained knowledge},
  author={Tang, Pengwei and Hu, Xiaolin and Liu, Yong and Ding, Lizhong and Zhang, Dongjie and Wu, Xing and Zhang, Debing},
  booktitle={Proceedings of the AAAI Conference on Artificial Intelligence},
  volume={40},
  number={39},
  pages={33232--33240},
  year={2026}
}

@inproceedings{oplora,
  title={Oplora: Orthogonal projection lora prevents catastrophic forgetting during parameter-efficient fine-tuning},
  author={Xiong, Yifeng and Xie, Xiaohui},
  booktitle={Proceedings of the AAAI Conference on Artificial Intelligence},
  volume={40},
  number={40},
  pages={34088--34096},
  year={2026}
}

@article{nqopen,
  title={Natural Questions: A Benchmark for Question Answering Research},
  author={Tom Kwiatkowski and Jennimaria Palomaki and Olivia Redfield and Michael Collins and Ankur P. Parikh and Chris Alberti and Danielle Epstein and Illia Polosukhin and Jacob Devlin and Kenton Lee and Kristina Toutanova and Llion Jones and Matthew Kelcey and Ming-Wei Chang and Andrew M. Dai and Jakob Uszkoreit and Quoc V. Le and Slav Petrov},
  journal={Transactions of the Association for Computational Linguistics},
  year={2019},
  volume={7},
  pages={453-466},
  url={https://api.semanticscholar.org/CorpusID:86611921}
}

@article{triviaqa,
  title={TriviaQA: A Large Scale Distantly Supervised Challenge Dataset for Reading Comprehension},
  author={Mandar Joshi and Eunsol Choi and Daniel S. Weld and Luke Zettlemoyer},
  journal={ArXiv},
  year={2017},
  volume={abs/1705.03551},
  url={https://api.semanticscholar.org/CorpusID:26501419}
}

@inproceedings{webqs,
  author       = {Jonathan Berant and
                  Andrew Chou and
                  Roy Frostig and
                  Percy Liang},
  title        = {Semantic Parsing on Freebase from Question-Answer Pairs},
  booktitle    = {Proceedings of the 2013 Conference on Empirical Methods in Natural
                  Language Processing, {EMNLP} 2013, 18-21 October 2013, Grand Hyatt
                  Seattle, Seattle, Washington, USA, {A} meeting of SIGDAT, a Special
                  Interest Group of the {ACL}},
  pages        = {1533--1544},
  publisher    = {{ACL}},
  year         = {2013},
  url          = {https://doi.org/10.18653/v1/d13-1160},
  doi          = {10.18653/V1/D13-1160},
  bibsource    = {dblp computer science bibliography, https://dblp.org}
}

@article{metamath,
  title={Metamath: Bootstrap your own mathematical questions for large language models},
  author={Yu, Longhui and Jiang, Weisen and Shi, Han and Yu, Jincheng and Liu, Zhengying and Zhang, Yu and Kwok, James T and Li, Zhenguo and Weller, Adrian and Liu, Weiyang},
  journal={arXiv preprint arXiv:2309.12284},
  year={2023}
}

@article{gsm8k,
  author       = {Karl Cobbe and
                  Vineet Kosaraju and
                  Mohammad Bavarian and
                  Mark Chen and
                  Heewoo Jun and
                  Lukasz Kaiser and
                  Matthias Plappert and
                  Jerry Tworek and
                  Jacob Hilton and
                  Reiichiro Nakano and
                  Christopher Hesse and
                  John Schulman},
  title        = {Training Verifiers to Solve Math Word Problems},
  journal      = {CoRR},
  volume       = {abs/2110.14168},
  year         = {2021},
  url          = {https://arxiv.org/abs/2110.14168},
  eprinttype   = {arXiv},
  eprint       = {2110.14168},
  bibsource    = {dblp computer science bibliography, https://dblp.org}
}

@article{math,
  title={Measuring mathematical problem solving with the math dataset},
  author={Hendrycks, Dan and Burns, Collin and Kadavath, Saurav and Arora, Akul and Basart, Steven and Tang, Eric and Song, Dawn and Steinhardt, Jacob},
  journal={arXiv preprint arXiv:2103.03874},
  year={2021}
}

@article{humaneval,
  title={Evaluating Large Language Models Trained on Code},
  author={Mark Chen and Jerry Tworek and Heewoo Jun and Qiming Yuan and Henrique Pond{\'e} and Jared Kaplan and Harrison Edwards and Yura Burda and Nicholas Joseph and Greg Brockman and Alex Ray and Raul Puri and Gretchen Krueger and Michael Petrov and Heidy Khlaaf and Girish Sastry and Pamela Mishkin and Brooke Chan and Scott Gray and Nick Ryder and Mikhail Pavlov and Alethea Power and Lukasz Kaiser and Mo Bavarian and Clemens Winter and Phil Tillet and Felipe Petroski Such and David W. Cummings and Matthias Plappert and Fotios Chantzis and Elizabeth Barnes and Ariel Herbert-Voss and William H. Guss and Alex Nichol and Igor Babuschkin and Suchir Balaji and Shantanu Jain and Andrew Carr and Jan Leike and Josh Achiam and Vedant Misra and Evan Morikawa and Alec Radford and Matthew M. Knight and Miles Brundage and Mira Murati and Katie Mayer and Peter Welinder and Bob McGrew and Dario Amodei and Sam McCandlish and Ilya Sutskever and Wojciech Zaremba},
  journal={ArXiv},
  year={2021},
  volume={abs/2107.03374},
  url={https://api.semanticscholar.org/CorpusID:235755472}
}

@article{mbpp,
  title={Program synthesis with large language models},
  author={Austin, Jacob and Odena, Augustus and Nye, Maxwell and Bosma, Maarten and Michalewski, Henryk and Dohan, David and Jiang, Ellen and Cai, Carrie and Terry, Michael and Le, Quoc and others},
  journal={arXiv preprint arXiv:2108.07732},
  year={2021}
}

@inproceedings{codefeedback,
  title={Opencodeinterpreter: Integrating code generation with execution and refinement},
  author={Zheng, Tianyu and Zhang, Ge and Shen, Tianhao and Liu, Xueling and Lin, Bill Yuchen and Fu, Jie and Chen, Wenhu and Yue, Xiang},
  booktitle={Findings of the Association for Computational Linguistics: ACL 2024},
  pages={12834--12859},
  year={2024}
}

@inproceedings{wizard,
  title={WizardLM: Empowering Large Pre-Trained Language Models to Follow Complex Instructions},
  author={Can Xu and Qingfeng Sun and Kai Zheng and Xiubo Geng and Pu Zhao and Jiazhan Feng and Chongyang Tao and Daxin Jiang},
  booktitle={International Conference on Learning Representations},
  year={2023},
  url={https://api.semanticscholar.org/CorpusID:258298159}
}

@article{ifeval,
  title={Instruction-following evaluation for large language models},
  author={Zhou, Jeffrey and Lu, Tianjian and Mishra, Swaroop and Brahma, Siddhartha and Basu, Sujoy and Luan, Yi and Zhou, Denny and Hou, Le},
  journal={arXiv preprint arXiv:2311.07911},
  year={2023}
}

@article{lora,
  title={Lora: Low-rank adaptation of large language models.},
  author={Hu, Edward J and Shen, Yelong and Wallis, Phillip and Allen-Zhu, Zeyuan and Li, Yuanzhi and Wang, Shean and Wang, Liang and Chen, Weizhu and others},
  journal={Iclr},
  volume={1},
  number={2},
  pages={3},
  year={2022}
}

@article{llama2,
  title={Llama 2: Open foundation and fine-tuned chat models},
  author={Touvron, Hugo and Martin, Louis and Stone, Kevin and Albert, Peter and Almahairi, Amjad and Babaei, Yasmine and Bashlykov, Nikolay and Batra, Soumya and Bhargava, Prajjwal and Bhosale, Shruti and others},
  journal={arXiv preprint arXiv:2307.09288},
  year={2023}
}

@inproceedings{preference,
  author       = {Rafael Rafailov and
                  Archit Sharma and
                  Eric Mitchell and
                  Christopher D. Manning and
                  Stefano Ermon and
                  Chelsea Finn},
  editor       = {Alice Oh and
                  Tristan Naumann and
                  Amir Globerson and
                  Kate Saenko and
                  Moritz Hardt and
                  Sergey Levine},
  title        = {Direct Preference Optimization: Your Language Model is Secretly a
                  Reward Model},
  booktitle    = {Advances in Neural Information Processing Systems 36: Annual Conference
                  on Neural Information Processing Systems 2023, NeurIPS 2023, New Orleans,
                  LA, USA, December 10 - 16, 2023},
  year         = {2023},
  url          = {http://papers.nips.cc/paper\_files/paper/2023/hash/a85b405ed65c6477a4fe8302b5e06ce7-Abstract-Conference.html},
  bibsource    = {dblp computer science bibliography, https://dblp.org}
}

@inproceedings{squad,
  title={Squad: 100,000+ questions for machine comprehension of text},
  author={Rajpurkar, Pranav and Zhang, Jian and Lopyrev, Konstantin and Liang, Percy},
  booktitle={Proceedings of the 2016 conference on empirical methods in natural language processing},
  pages={2383--2392},
  year={2016}
}

@inproceedings{mmlu,
  author       = {Dan Hendrycks and
                  Collin Burns and
                  Steven Basart and
                  Andy Zou and
                  Mantas Mazeika and
                  Dawn Song and
                  Jacob Steinhardt},
  title        = {Measuring Massive Multitask Language Understanding},
  booktitle    = {9th International Conference on Learning Representations, {ICLR} 2021,
                  Virtual Event, Austria, May 3-7, 2021},
  publisher    = {OpenReview.net},
  year         = {2021},
  url          = {https://openreview.net/forum?id=d7KBjmI3GmQ},
  bibsource    = {dblp computer science bibliography, https://dblp.org}
}

@inproceedings{truthfulqa,
  title={Truthfulqa: Measuring how models mimic human falsehoods},
  author={Lin, Stephanie and Hilton, Jacob and Evans, Owain},
  booktitle={Proceedings of the 60th annual meeting of the association for computational linguistics (volume 1: long papers)},
  pages={3214--3252},
  year={2022}
}

@article{babilong,
  title={Babilong: Testing the limits of llms with long context reasoning-in-a-haystack},
  author={Kuratov, Yuri and Bulatov, Aydar and Anokhin, Petr and Rodkin, Ivan and Sorokin, Dmitry and Sorokin, Artyom and Burtsev, Mikhail},
  journal={Advances in Neural Information Processing Systems},
  volume={37},
  pages={106519--106554},
  year={2024}
}

@inproceedings{longbenchv2,
  title={Longbench v2: Towards deeper understanding and reasoning on realistic long-context multitasks},
  author={Bai, Yushi and Tu, Shangqing and Zhang, Jiajie and Peng, Hao and Wang, Xiaozhi and Lv, Xin and Cao, Shulin and Xu, Jiazheng and Hou, Lei and Dong, Yuxiao and others},
  booktitle={Proceedings of the 63rd Annual Meeting of the Association for Computational Linguistics (Volume 1: Long Papers)},
  pages={3639--3664},
  year={2025}
}

@article{gpqa,
  author       = {David Rein and
                  Betty Li Hou and
                  Asa Cooper Stickland and
                  Jackson Petty and
                  Richard Yuanzhe Pang and
                  Julien Dirani and
                  Julian Michael and
                  Samuel R. Bowman},
  title        = {{GPQA:} {A} Graduate-Level Google-Proof Q{\&}A Benchmark},
  journal      = {CoRR},
  volume       = {abs/2311.12022},
  year         = {2023},
  url          = {https://doi.org/10.48550/arXiv.2311.12022},
  doi          = {10.48550/ARXIV.2311.12022},
  eprinttype   = {arXiv},
  eprint       = {2311.12022},
  bibsource    = {dblp computer science bibliography, https://dblp.org}
}

@inproceedings{qwen_report,
  title={Qwen3 Technical Report},
  author={An Yang and Anfeng Li and Baosong Yang and Beichen Zhang and Binyuan Hui and Bo Zheng and Bowen Yu and Chang Gao and Chengen Huang and Chenxu Lv and Chujie Zheng and Dayiheng Liu and Fan Zhou and Fei Huang and Feng Hu and Hao Ge and Haoran Wei and Huan Lin and Jialong Tang and Jian Yang and Jianhong Tu and Jianwei Zhang and Jianxin Yang and Jiaxin Yang and Jingren Zhou and Jingren Zhou and Junyan Lin and Kai Dang and Keqin Bao and Ke‐Pei Yang and Le Yu and Li-Chun Deng and Mei Li and Min Xue and Mingze Li and Pei Zhang and Peng Wang and Qin Zhu and Rui Men and Ruize Gao and Shixuan Liu and Shuang Luo and Tianhao Li and Tianyi Tang and Wenbiao Yin and Xingzhang Ren and Xinyu Wang and Xinyu Zhang and Xuancheng Ren and Yang Fan and Yang Su and Yi-Chao Zhang and Yinger Zhang and Yu Wan and Yuqiong Liu and Zekun Wang and Zeyu Cui and Zhenru Zhang and Zhipeng Zhou and Zihan Qiu},
  year={2025},
  url={https://api.semanticscholar.org/CorpusID:278602855}
}

@article{pca,
  title={Principal Component Analysis},
  author={Ian T. Jolliffe},
  journal={Technometrics},
  year={2003},
  volume={45},
  pages={276 - 276},
  url={https://api.semanticscholar.org/CorpusID:2534141}
}

@inproceedings{jiang2024approximation,
  author       = {Haotian Jiang and
                  Qianxiao Li},
  editor       = {Amir Globersons and
                  Lester Mackey and
                  Danielle Belgrave and
                  Angela Fan and
                  Ulrich Paquet and
                  Jakub M. Tomczak and
                  Cheng Zhang},
  title        = {Approximation Rate of the Transformer Architecture for Sequence Modeling},
  booktitle    = {Advances in Neural Information Processing Systems 38: Annual Conference
                  on Neural Information Processing Systems 2024, NeurIPS 2024, Vancouver,
                  BC, Canada, December 10 - 15, 2024},
  year         = {2024},
  url          = {http://papers.nips.cc/paper\_files/paper/2024/hash/7f64034009f4a5fa417a57e1a987c5cd-Abstract-Conference.html},
  bibsource    = {dblp computer science bibliography, https://dblp.org}
}

@inproceedings{forgetting5,
  author       = {Ananya Kumar and
                  Aditi Raghunathan and
                  Robbie Matthew Jones and
                  Tengyu Ma and
                  Percy Liang},
  title        = {Fine-Tuning can Distort Pretrained Features and Underperform Out-of-Distribution},
  booktitle    = {The Tenth International Conference on Learning Representations, {ICLR}
                  2022, Virtual Event, April 25-29, 2022},
  publisher    = {OpenReview.net},
  year         = {2022},
  url          = {https://openreview.net/forum?id=UYneFzXSJWh},
  bibsource    = {dblp computer science bibliography, https://dblp.org}
}

@inproceedings{grc,
  title={Rayleigh Quotient Based Optimization Methods For Eigenvalue Problems},
  author={Ren-Cang Li},
  year={2014},
  url={https://api.semanticscholar.org/CorpusID:14329440}
}

@article{uat,
  title={Are Transformers universal approximators of sequence-to-sequence functions?},
  author={Chulhee Yun and Srinadh Bhojanapalli and Ankit Singh Rawat and Sashank J. Reddi and Sanjiv Kumar},
  journal={ArXiv},
  year={2019},
  volume={abs/1912.10077},
  url={https://api.semanticscholar.org/CorpusID:209444410}
}

@article{yu2026attentionheadcount,
  author       = {Penghao Yu and
                  Haotian Jiang and
                  Zeyu Bao and
                  Ruoxi Yu and
                  Qianxiao Li},
  title        = {The Effect of Attention Head Count on Transformer Approximation},
  journal      = {CoRR},
  volume       = {abs/2510.06662},
  year         = {2025},
  url          = {https://doi.org/10.48550/arXiv.2510.06662},
  doi          = {10.48550/ARXIV.2510.06662},
  eprinttype   = {arXiv},
  eprint       = {2510.06662},
  bibsource    = {dblp computer science bibliography, https://dblp.org}
}

@inproceedings{adamw,
  title={Decoupled Weight Decay Regularization},
  author={Ilya Loshchilov and Frank Hutter},
  booktitle={International Conference on Learning Representations},
  year={2017},
  url={https://api.semanticscholar.org/CorpusID:53592270}
}

@article{glm,
  title={Glm-5: from vibe coding to agentic engineering},
  author={Zeng, Aohan and Lv, Xin and Hou, Zhenyu and Du, Zhengxiao and Zheng, Qinkai and Chen, Bin and Yin, Da and Ge, Chendi and Huang, Chenghua and Xie, Chengxing and others},
  journal={arXiv preprint arXiv:2602.15763},
  year={2026}
}

@article{deepseek,
  title={Deepseek-r1: Incentivizing reasoning capability in llms via reinforcement learning},
  author={Guo, Daya and Yang, Dejian and Zhang, Haowei and Song, Junxiao and Wang, Peiyi and Zhu, Qihao and Xu, Runxin and Zhang, Ruoyu and Ma, Shirong and Bi, Xiao and others},
  journal={arXiv preprint arXiv:2501.12948},
  year={2025}
}

@inproceedings{dpo,
  author       = {Rafael Rafailov and
                  Archit Sharma and
                  Eric Mitchell and
                  Christopher D. Manning and
                  Stefano Ermon and
                  Chelsea Finn},
  editor       = {Alice Oh and
                  Tristan Naumann and
                  Amir Globerson and
                  Kate Saenko and
                  Moritz Hardt and
                  Sergey Levine},
  title        = {Direct Preference Optimization: Your Language Model is Secretly a
                  Reward Model},
  booktitle    = {Advances in Neural Information Processing Systems 36: Annual Conference
                  on Neural Information Processing Systems 2023, NeurIPS 2023, New Orleans,
                  LA, USA, December 10 - 16, 2023},
  year         = {2023},
  url          = {http://papers.nips.cc/paper\_files/paper/2023/hash/a85b405ed65c6477a4fe8302b5e06ce7-Abstract-Conference.html},
  bibsource    = {dblp computer science bibliography, https://dblp.org}
}

@inproceedings{opd,
  title={On-Policy Distillation of Language Models: Learning from Self-Generated Mistakes},
  author={Rishabh Agarwal and Nino Vieillard and Yongchao Zhou and Piotr Stańczyk and Sabela Ramos and Matthieu Geist and Olivier Bachem},
  booktitle={International Conference on Learning Representations},
  year={2023},
  url={https://api.semanticscholar.org/CorpusID:263610088}
}

@article{kimi,
  title={Kimi K2. 5: Visual Agentic Intelligence},
  author={Team, Kimi and Bai, Tongtong and Bai, Yifan and Bao, Yiping and Cai, SH and Cao, Yuan and Charles, Y and Che, HS and Chen, Cheng and Chen, Guanduo and others},
  journal={arXiv preprint arXiv:2602.02276},
  year={2026}
}

@inproceedings{v4,
  title={DeepSeek-V4: Towards Highly Efficient Million-Token Context Intelligence},
  author={DeepSeek-AI and Anyi Xu and Bang Lin and Bing Xue and Bing-Li Wang and Bin Xu and Bo Wu and Bowei Zhang and Chao Lin and Chengyao Dong and Chen Ling and Chengda Lu and Chen Zhao and Chengqi Deng and Cheng Hou and Chen Xu and Chenze Shao and Chong Ruan and Conner Sun and Damai Dai and Daya Guo and De-Bin Yang and Deli Chen and Dong-Hui Li and Dong-Li Ji and Erhang Li and Fangchen Wei and Fangyun Lin and Fang Yuan and Fei Xia and Fucong Dai and Guangbo Hao and Guanting Chen and Guo Cao and Guo-Hui Meng and Guowei Li and Han Yu and Han Zhang and Hanwei Xu and Hao Li and Hao Liang and Haoling Zhang and Haoming Luo and Haoran Wei and Hao Yuan and Haowei Zhang and Haowen Luo and Hao Chen and Haozhe Ji and Heng Zhang and Honghui Ding and Hongxuan Tang and Huanqi Cao and Huazuo Gao and Hui Qu and Hui Zeng and J. Yang and JQ Zhu and Jianming Luo and Jia Song and Jia Yu and Jialiang Huang and Jialu Cai and Jian Liang and Jiangting Zhou and Jiasheng Ye and Jiashi Li and Jiaxin Xu and Jiewen Hu and Jieyu Yang and Jin Chen and Jin Yan and JingChang Chen and Jing Zhou and Jing Xiang and Jingyang Yuan and Jing Cheng and Jingzi Zhou and Jinhua Zhu and Jiping Yu and Joseph Sun and Junliang Ran and Jun Jiang and Junjie Qiu and Junlong Li and Junming Zheng and Jun-Mei Song and Kai Dong and Kaige Gao and Kang Guan and Kexing Zhou and Kezhao Huang and Kuai Yu and Lean Wang and Lecong Zhang and Lei Wang and Leyi Xia and Li Zhang and Liang Zhao and Lihua Guo and Lin Luo and Lin Ma and Linyan Zhu and Litong Wang and Liyu Cai and Liyue Zhang and Longhao Chen and Mingxi Di and My Xu and Max Mei and Miaojun Wang and Mingchuan Zhang and Minghua Zhang and Minghui Tang and Mingming Li and Mi Zhou and Min Han and Ning Wang and Pan Huang and Pan Wang and Peixin Cong and Peiyi Wang and Peng Zhang and Qiancheng Wang and Qihao Zhu and Qingyang Li and Qinyu Chen and Qiushi Du and Qi Jiang and Rui Tian and Rui-qin Xu and Ruijie Lu and Ruiling Xu and Ruiqi Ge and Ruisong Zhang and Ruizhe Pan and Runji Wang and Run Da Chen and Runqiu Yin and Runxin Xu and Ruo-Han Shen and Ruoyu Zhang and Ruyi Chen and Sh. Liu and Shanghao Lu and Shangmian Sun and Shangyan Zhou and Shan-Shan Chen and Shaofei Cai and Shao Nie and Shao-Ping Wu and Shaoyuan Chen and Shengding Hu and Sheng Liu and Shiqiang Hu and Shirong Ma and Shiyu Wang and Shuiping Yu and Shunfeng Zhou and Shuting Pan and Shuying Yu and Songyang Zhou and Tao Ni and Tao Yun and Tian Jin and Tianhong Pei and Tian Ye and Tianle Lin and Tian Ji and Tianyi Cui and Tianyuan Yue and Tingting Yu and Tu-Tu Wang and W. Y. Zhang and Wl Xiao and Wangding Zeng and Wei An and Weilin Zhao and Wen Liu and Wenfeng Liang and Wenjie Pang and Wenjing Luo and Wenjin Yao and Wenjun Gao and Wenkai Yang and Wen-Wen Huang and Wenqing Hou and Wentao Zhang and Wenting Ma and Xi Gao and Xiang He and Xiang Wang and Xianzu Wang and Xiao Bi and Xiaodong Liu and Xiaohan Wang and Xiaokang Chen and Xiaokang Zhang and Xiaotao Nie and Xiaowen Sun and Xiaoxiang Wang and Xin Cheng and Xin Liu and Xin Xie and Xingchao Liu and Xingchen Liu and Xingkai Yu and Xingyou Li and Xinyu Yang and Xinyu Zhang and Xu Chen and Xuanyu Wang and Xuecheng Su and Xueyin Chen and Xuheng Lin and Xu Fu and Yc Yan and YQ Wang and Yw Ma and Yanfen Luo and Yang Zhang and Yanhong Xu and Yanru Ma and Yanwen Huang and Yao Li and Yao Xu and Yao Zhao and Yaofeng Sun and Yaohui Wang and Yi Qian and Yingxia Shao and Yi Yu and Yichao Zhang and Yifan Ding and Yifan Shi and Yijia Wu and Yi-Yu Xiong and Yi Ma and Ying He and Ying Tang and Ying Zhou and Yingjia Luo and Yinmin Zhong and Yishi Piao and Yisong Wang and Yixiang Zhang and Yixiao Chen and Yixuan Tan and Yixuan Wei and Yiyang Ma and Yiyuan Liu and Yong Yang and Yongqiang Guo and Yongtong Wu and Yu Wu and Yukun Li and Yuan Cheng and Yuan Ou and Yuanfan Xu and Yuanhao Li and Yuduan Wang and Yuehan Yang and Yue Xu and Yuhan Wu and Yu Meng and Yu Zou and Yukun Zha and Yunfan Xiong and Yupeng Chen and Yuping Lin and Yu Cao and Yuqian Wang and Yushun Zhang and Yuting Yan and Yutong Lin and Yuxian Gu and Yu-Wei Luo and Yu-mei You and Yuxuan Liu and Yuxuan Zhou and Yuyang Zhou and Yuzhen Huang and Zhanghua Wu and Ze Wang and Ze Zhao and Zehui Ren and Zekai Zhang and Zhangli Sha and Zhe Fu and Zhe Ju and Zhean Xu and Zhenda Xie and Zhengyan Zhang and Zhe Gao and Zhewen Hao and Zhibin Gou and Zhicheng Ma and Zhigang Yan and Zhihong Shao and Zhixiang Huang and Zhixuan Chen and Zhiyu Wu and Zhizhou Ren and Zhongyu Wu and Zhuoshu Li and Zhuping Zhang and Zian Xu and Zihao Wang and Zihua Qu and Zihui Gu and Zijia Zhu and Zi-Long Li and Zipeng Zhang and Ziwei Xie and Ziyi Gao and Ziyi Wan and Zizheng Pan and Zong-xiang Yao},
  year={2026},
  url={https://api.semanticscholar.org/CorpusID:289623513}
}

@article{rlsft,
  title={RL's Razor: Why Online Reinforcement Learning Forgets Less},
  author={Idan Shenfeld and Jyothish Pari and Pulkit Agrawal},
  journal={ArXiv},
  year={2025},
  volume={abs/2509.04259},
  url={https://api.semanticscholar.org/CorpusID:281103647}
}

@inproceedings{dynamic,
  title={Understanding catastrophic forgetting in language models via implicit inference},
  author={Kotha, Suhas and Springer, Jacob and Raghunathan, Aditi},
  booktitle={International Conference on Learning Representations},
  volume={2024},
  pages={24110--24139},
  year={2024}
}

@inproceedings{sharpness,
  title={Revisiting catastrophic forgetting in large language model tuning},
  author={Li, Hongyu and Ding, Liang and Fang, Meng and Tao, Dacheng},
  booktitle={Findings of the association for computational linguistics: EMNLP 2024},
  pages={4297--4308},
  year={2024}
}

@inproceedings{
mapping,
title={Mapping Post-Training Forgetting in Language Models at Scale},
author={Jackson Harmon and Andreas Hochlehnert and Matthias Bethge and Ameya Prabhu},
booktitle={The Fourteenth International Conference on Learning Representations},
year={2026},
url={https://openreview.net/forum?id=qCIg2WGudx}
}

@inproceedings{
razor,
title={{RL}'s Razor: Why Online Reinforcement Learning Forgets Less},
author={Idan Shenfeld and Jyothish Pari and Pulkit Agrawal},
booktitle={The Fourteenth International Conference on Learning Representations},
year={2026},
url={https://openreview.net/forum?id=7HNRYT4V44}
}

@inproceedings{
sdft,
title={Self-Distillation Enables Continual Learning},
author={Idan Shenfeld and Mehul Damani and Jonas H{\"u}botter and Pulkit Agrawal},
booktitle={Forty-third International Conference on Machine Learning},
year={2026},
url={https://openreview.net/forum?id=qA6FgH0nnZ}
}

@article{locating,
  title={Small singular values matter: A random matrix analysis of transformer models},
  author={Staats, Max and Thamm, Matthias and Rosenow, Bernd},
  journal={Advances in Neural Information Processing Systems},
  volume={38},
  pages={153545--153573},
  year={2026}
}

@article{dilema1,
  title={Homeostatic plasticity in the developing nervous system},
  author={Gina G. Turrigiano and Sacha B. Nelson},
  journal={Nature Reviews Neuroscience},
  year={2004},
  volume={5},
  pages={97-107},
  url={https://api.semanticscholar.org/CorpusID:14535839}
}

@article{dilema2,
  title={The ART of adaptive pattern recognition by a self-organizing neural network},
  author={Gail A. Carpenter and Stephen Grossberg},
  journal={Computer},
  year={1988},
  volume={21},
  pages={77-88},
  url={https://api.semanticscholar.org/CorpusID:14625094}
}

@article{bt,
  title={Beyond not-forgetting: Continual learning with backward knowledge transfer},
  author={Lin, Sen and Yang, Li and Fan, Deliang and Zhang, Junshan},
  journal={Advances in Neural Information Processing Systems},
  volume={35},
  pages={16165--16177},
  year={2022}
}
\bibliographystyle{style}

\newpage
\appendix

\par\noindent\rule{\textwidth}{2pt}
\begin{table}[h]
    \vspace{-0.3cm}
    \centering
    \Large
    \begin{tabular}{c}
\textbf{Appendix}
    \end{tabular}
    \vspace{-.5cm}
\end{table}
\par\noindent\rule{\textwidth}{1pt}

\tableofcontents
\addtocontents{toc}{\protect\setcounter{tocdepth}{2}}

\newpage


\section{Experimental Details}
\label{appen:experiments}
\subsection{Training Details and Hyperparameters} 
For a fair comparison, we follow the same optimization protocol across all methods. Specifically, we use the AdamW optimizer with a batch size of 128 and a learning rate of $2\times 10^{-5}$. We adopt a cosine annealing schedule with a warm-up ratio of 0.03. $\beta_1,\beta_2,\epsilon$ values in the AdamW optimizer are set to 0.9, 0.999, and $1\times 10^{-8}$, respectively. Training is performed for one epoch on the first 100,000 conversations of each target dataset, and the loss is computed on the response tokens. The maximum training sequence length is 512; the maximum gradient norm is set to 1. Throughout the training, we use bf16 precision; weight decay and dropout are not applied. Our experiments are conducted on a single NVIDIA GH200 120GB GPU. We set the per-device training batch size to 16 and the gradient accumulation steps to 8. Since training runs on a single GPU, each optimizer update uses an effective batch size of 128 samples. The LoRA rank is set to 128 for OGE and all LoRA-based baselines. For the OPLoRA~\cite{oplora} baseline, we use a projection rank of $k = 128$, which is the best operating point according to the original paper. For OGE, we consistently use the hyperparameters $\lambda = 1$ and $s_{\mathrm{min}} = 0.3$, unless otherwise specified. As shown in Fig.~\ref{fig:OGE_ablation}(d), these hyperparameters can be efficiently tuned through a modest one-dimensional search along each axis.

The fine-tuning prompt is as follows:

\begin{verbatim}
PROMPT = (
    "Below is an instruction that "
    "describes a task. "
    "Write a response that "
    "appropriately completes the "
    "request.\n\n"
    "### Instruction:\n"
    "{instruction}\n\n"
    "### Response:"
)
\end{verbatim}
The above setup is applied consistently across all baseline experiments.

\subsection{Covariance Matrix Construction From Block Input}
CorDA~\cite{corda}, LoRA-Null~\cite{loranull}, and OGE (ours) methods need to construct the covariance matrices ($C_\mathrm{pre} = X_\mathrm{pre}X_\mathrm{pre}^\top, C_\mathrm{task}= X_\mathrm{task}X_\mathrm{task}^\top$) using each layer's input representations. For constructing the model-sampled proxy set used to estimate $C_{\mathrm{pre}}$, we sample 2,000 unconditionally generated outputs from the model. When constructing the covariance matrix from a single dataset (e.g., baseline approach, or $C_\mathrm{task}$ construction of OGE), we use 256 query samples from the dataset. For numerical stability, we regularize the pretraining covariance before whitening and solving the generalized eigenvalue problem. Specifically, OGE uses ${C}_{\mathrm{pre}} \leftarrow C_{\mathrm{pre}}+\delta I$ with $\delta=0.01$ throughout all experiments. Visualized layer statistics (e.g., Fig.~\ref{fig:first_fig} and Fig.~\ref{fig:OGE_ablation}(a)) are extracted from the MLP down-projection matrix of the 15th transformer block.

\subsection{LoRA Adaptation of OGE}
\label{app:OGE_init}
Under the LoRA parameterization, \(W=W_0+BA\), we have \(A=A_0+\Delta A\) and \(B=B_0+\Delta B\), where \(A_0,B_0\) denote the initial LoRA factors and \(\Delta A,\Delta B\) are their updates. Assuming the initial adapter does not perturb the pretrained model (e.g., \(B_0A_0=0\)), the weight perturbation satisfies
{\small
\begin{equation}
\label{eq:lora_init}
\Delta W X_{\mathrm{pre}} = (\Delta B A_0) X_{\mathrm{pre}}+(B_0+\Delta B)\Delta A  X_{\mathrm{pre}}.
\end{equation}
}
\(A_0\) matrix can be initialized in the null space of \(X_{\mathrm{pre}}\), then \(A_0X_{\mathrm{pre}}=0\), and hence \((\Delta B A_0)X_{\mathrm{pre}}=0\). The perturbation on pretraining-proxy activations therefore reduces to \((B_0+\Delta B)\Delta A X_{\mathrm{pre}}\).  Since
{\small
\begin{equation}
\label{eq:drift_eq}
\|(B_0+\Delta B)\Delta A X_{\mathrm{pre}}\|_F
\le
\|B_0+\Delta B\|_2\,\|\Delta A X_{\mathrm{pre}}\|_F,
\end{equation}
}
we can limit this drift by controlling \(\|\Delta A X_{\mathrm{pre}}\|_F^2\). Now, we can define adaptation utility and forgetting risk as used in the main paper:
\begin{align}
    \mathcal R_{\mathrm{pre}}(\Delta A)
&:=
\|\Delta A X_{\mathrm{pre}}\|_F^2, \\ \mathcal U_{\mathrm{task}}(\Delta A)
&:=
\|\Delta A X_{\mathrm{task}}\|_F^2
\end{align}
Thus, OGE applies spectral-gated Adam (SAdam) to the down-projection matrix \(A\), equivalently to the update \(\Delta A=A-A_0\), while the up-projection matrix
\(B\) is trained with standard AdamW in the original coordinates. Thus, the training trajectory of down-projection $A$ is guided by the SAdam optimizer, while up-projection $B$ is optimized without explicit constraint. 

\subsection{Detailed Configuration of Model-generated Calibration Dataset}
We sample $N=2{,}000$ unconditional generations from the base model in
\texttt{float16} on a single GPU with a single random seed. Each generation is conditioned
on a single BOS and produced by ancestral sampling
with $\ell_{\min}=128$ and $\ell_{\max}=512$ new tokens (\texttt{eos} permitted
only after $\ell_{\min}$). To balance fidelity to $p_{\theta}$ with coverage,
the decoding hyperparameters for each sample are drawn from a fixed
three-component mixture:
(i) with probability $0.70$, $T=1.0$, top-$p=1.0$, top-$k=0$ (closest to
$p_{\theta}$);
(ii) with probability $0.20$, $T=1.0$, top-$p=0.95$, top-$k=0$ (mild
truncation for stability);
(iii) with probability $0.10$, $T=1.2$, top-$p=0.98$, top-$k=50$ (broader
exploration).
After generation, samples are filtered by (a) exact-text duplication (SHA-1) and (b) a degeneracy screen that rejects sequences shorter than $32$
tokens, with a run of $\ge 20$ identical consecutive tokens, unique-token ratio
below $0.12$ at length $\ge 128$, or unique-$4$-gram ratio below $0.25$ at
length $\ge 256$. Sampling is repeated until $2{,}000$ samples pass both
filters.

\subsection{Benchmark Details}
\label{app:benchmark-details}
\begin{table*}[t]
\centering
\small
\setlength{\tabcolsep}{4pt}
\renewcommand{\arraystretch}{1.12}
\caption{
Summary of the datasets and benchmarks used in the main experiments.
\textit{Train} denotes downstream supervised fine-tuning data,
\textit{Adapt.} denotes target-task evaluation,
\textit{Pres.} denotes non-target foundation-preservation evaluation.
}
\label{tab:dataset_summary}
\begin{tabularx}{\textwidth}{
    @{}
    >{\raggedright\arraybackslash}p{2.25cm}
    >{\raggedright\arraybackslash}p{1.20cm}
    >{\raggedright\arraybackslash}p{2.05cm}
    >{\raggedright\arraybackslash}p{3.45cm}
    >{\raggedright\arraybackslash}X
    @{}
}
\toprule
\textbf{Experimental setting}
&
\textbf{Role}
&
\textbf{Domain}
&
\textbf{Dataset / benchmark}
&
\textbf{Purpose and reported metric}
\\
\midrule

\multirow{7}{*}{\shortstack{LLaMA2-7B-Base\\main evaluation}}
&
Train
&
Mathematics
&
MetaMathQA
&
Supervised fine-tuning data for mathematical-reasoning adaptation.
\\

&
Adapt.
&
Mathematics
&
GSM8K; MATH
&
Evaluation of grade-school and competition-level mathematical reasoning;
task-specific scores are averaged as Avg2.
\\

&
Train
&
Code
&
CodeFeedback
&
Supervised fine-tuning data for code-generation adaptation.
\\

&
Adapt.
&
Code
&
HumanEval; MBPP
&
Evaluation of functional correctness and program synthesis;
task-specific scores are averaged as Avg2.
\\

&
Train
&
Instruction following
&
WizardLM-Evol-Instruct
&
Supervised fine-tuning data for instruction-following adaptation.
\\

&
Adapt.
&
Instruction following
&
IFEVAL
&
Evaluation of compliance with explicit and verifiable instructions;
Prompt-Level Strict Accuracy is reported.
\\

&
Pres.
&
Factual retrieval
&
TriviaQA; NQ Open; WebQS
&
Evaluation of preserved world knowledge and retrieval-style reasoning;
exact-match scores, Avg1, and Avg1(\%) are reported.
\\

\midrule

\multirow{9}{*}{\shortstack{Qwen3-1.7B-Base\\evaluation}}
&
Train
&
Mathematics
&
MetaMathQA
&
Supervised fine-tuning data for the target math-adaptation task.
\\

&
Adapt.
&
Mathematics
&
GSM8K; MATH
&
Downstream mathematical-reasoning evaluation, reported as the aggregate
math-adaptation score.
\\

&
Pres.
&
Instruction following
&
IFEVAL
&
Preservation of instruction-following ability;
Prompt-Level Strict Accuracy is reported.
\\

&
Pres.
&
Truthfulness
&
TruthfulQA
&
Preservation of truthfulness and hallucination robustness.
\\

&
Pres.
&
Reading comprehension
&
SQuAD-v2
&
Preservation of passage-based reading comprehension;
HasAns Exact is reported.
\\

&
Pres.
&
Factual retrieval
&
TriviaQA; NQ Open
&
Preservation of open-domain factual knowledge;
exact-match scores are reported.
\\

&
Pres.
&
Multitask knowledge
&
MMLU
&
Preservation of broad academic and professional knowledge and reasoning.
\\

&
Pres.
&
Long-context modeling
&
BABILong; LongBench-v2
&
Preservation of long-context retrieval, integration, and reasoning.
\\

&
Pres.
&
Scientific reasoning
&
GPQA
&
Preservation of graduate-level scientific knowledge and reasoning.
\\

\bottomrule
\end{tabularx}

\vspace{1mm}
\begin{minipage}{\textwidth}
\footnotesize
\end{minipage}
\end{table*}
We evaluate each method along two axes: downstream adaptation and foundation preservation. 
Downstream adaptation benchmarks measure how well the fine-tuned model learns the target task, 
whereas preservation benchmarks measure whether non-target capabilities acquired during pretraining 
are retained after fine-tuning. All scores are reported as percentages, and higher values indicate 
better performance. Table~\ref{tab:dataset_summary} summarizes the training datasets and evaluation benchmarks used in our experiments.

\paragraph{Downstream adaptation benchmarks.}
For math adaptation, we fine-tune on MetaMathQA and evaluate on GSM8K and MATH. GSM8K consists 
of grade-school math word problems that require multi-step arithmetic reasoning, while MATH contains 
more challenging competition-style mathematical problems. For code adaptation, we fine-tune on 
CodeFeedback and evaluate on HumanEval and MBPP. HumanEval evaluates functional correctness on 
Python programming problems, and MBPP evaluates the ability to synthesize short programs from natural 
language descriptions. For instruction-following adaptation, we fine-tune on WizardLM-Evol-Instruct 
and evaluate on IFEVAL, which tests whether a model follows verifiable natural-language instructions. 
For each adaptation setting, we report the task-specific benchmark scores and use their arithmetic 
mean as Avg2 when multiple downstream benchmarks are used.

\paragraph{Foundation-preservation benchmarks.}
To measure preservation of pretrained capabilities, our main evaluation uses TriviaQA, NQ Open, 
and WebQS. These benchmarks test open-domain factual question answering and therefore provide a 
direct measure of whether fine-tuning preserves world knowledge and retrieval-style reasoning. 
We report exact-match scores on each benchmark. We also report Avg1, the arithmetic mean of the 
three preservation scores, and Avg1(\%), the ratio of Avg1 to the corresponding score of the 
original pretrained model. Avg1(\%) therefore measures how much of the pretrained model's 
non-target performance is retained after adaptation.

\paragraph{Broader preservation benchmarks.}
In addition to the main factual-retrieval benchmarks, we evaluate preservation on a broader 
benchmark suite in the Qwen3 experiments. Since the Qwen3 experiments fine-tune on MetaMathQA, all benchmarks 
in this suite are treated as non-target preservation benchmarks, while the math score is reported 
separately as the downstream adaptation metric.

The broader suite covers several complementary capabilities. IFEVAL evaluates instruction 
following by testing whether model outputs satisfy explicit and verifiable constraints in the 
prompt. In particular, we report the \textit{Prompt-Level Strict Accuracy} of IFEVAL, where forgetting behavior is more clearly manifested. TruthfulQA measures truthfulness and hallucination robustness by asking questions that 
often elicit common misconceptions or false answers. SQuAD-v2 evaluates reading comprehension, where 
the model must answer questions using evidence from a provided passage. Specifically, we report \textit{HasAns Exact} of SquAD-v2, which is the exact match score on the answerable data subset. TriviaQA and NQ Open 
measure open-domain factual question answering and retrieval-style world knowledge. MMLU evaluates 
broad multitask knowledge and reasoning across diverse academic and professional domains. BABILong 
and LongBench-v2 evaluate long-context understanding, including the ability to locate, integrate, 
and reason over information distributed across long inputs. GPQA evaluates difficult graduate-level 
scientific reasoning, providing a high-level test of whether fine-tuning preserves specialized 
reasoning ability beyond the target math adaptation task.

For this broader evaluation, we report each benchmark score individually and compute Avg1 as the 
arithmetic mean over all preservation benchmarks in the suite. This aggregate score summarizes 
overall foundation preservation across heterogeneous capabilities, while the separate benchmark 
scores reveal which capabilities are most affected by finetuning-induced forgetting. This is important because 
forgetting is not uniform across domains: a method may preserve factual retrieval while degrading 
instruction following, long-context reasoning, or scientific knowledge. The broader suite therefore 
provides a more comprehensive view of the preservation--adaptation trade-off than factual QA alone.

\section{Proof of Theorem~\ref{thm:kl-control}}
\label{app:theorem-proofs}
\paragraph{Roadmap.}
This appendix formalizes the connection between next-token prediction,
forward KL minimization, and Transformer approximation theory. We first show
that the population next-token prediction objective is exactly the forward KL
divergence from the data distribution to the model distribution, up to the
constant entropy of the data distribution. This converts population pretraining
into a forward KL minimization problem. We then derive an immediate
\(\varepsilon\)-optimality corollary, which separates the error into an
optimization term and a model-class misspecification term. The remaining
subsections control this misspecification term for Transformer models: first
qualitatively using a universal approximation argument for smoothed next-token
logits, and then quantitatively using a Jackson-type Transformer approximation
rate.
\subsection{Proof of Theorem~\ref{thm:ntp_forward_kl}} 
\begin{theorem}[Next-token prediction as forward KL minimization]
\label{thm:ntp_forward_kl}
Let $\mathcal X=\mathcal V^T$ be the discrete space of token sequences of length $T$ over a finite vocabulary $\mathcal V$, and let
$P=P_{\mathrm{data}}\in\Delta(\mathcal X)$ denote the data distribution. Let $\Theta$ be the parameter space of the autoregressive model family
$\{Q_\theta:\theta\in\Theta\}$. For each $\theta\in\Theta$, time step $t\in[T]$, and prefix $x_{<t}\in\mathcal V^{t-1}$, let
\begin{equation}
q_\theta(\cdot\mid x_{<t})\in\Delta(\mathcal V)
\end{equation}
be a conditional probability distribution over the next token. The resulting sequence-level distribution
$Q_\theta\in\Delta(\mathcal X)$ is defined by
\begin{align}
Q_\theta(x)
=
\prod_{t=1}^T q_\theta(x_t\mid x_{<t}),
\quad
x=(x_1,\dots,x_T)\in\mathcal X.
\end{align}
Assume $Q_\theta(x)>0$ whenever $P(x)>0$. Define the population next-token prediction loss
\begin{align}
\mathcal{L}(\theta)
:=
\mathbb E_{X\sim P}
\left[
-\sum_{t=1}^T
\log q_\theta(X_t\mid X_{<t})
\right].
\end{align}
Then
\begin{align}
\mathcal{L}(\theta)
=
\mathcal{H}(P)
+
D_{\mathrm{KL}}(P\Vert Q_\theta)
\end{align}
where
\begin{align}
\mathcal{H}(P):=-\sum_{x\in\mathcal X}P(x)\log P(x), \\
D_{\mathrm{KL}}(P\Vert Q_\theta)
:=
\sum_{x\in\mathcal X}
P(x)\log\frac{P(x)}{Q_\theta(x)}.
\end{align}
$\mathcal{H}$ is independent of $\theta$, optimization on next-token prediction directly minimize forward KL divergence:
\begin{align}
\boxed{
\arg\min_{\theta\in\Theta} \mathcal{L}(\theta)
=
\arg\min_{\theta\in\Theta}
D_{\mathrm{KL}}(P\Vert Q_\theta).
}
\end{align}
\end{theorem}
\begin{proof}
By the autoregressive factorization,
\begin{align}
\sum_{t=1}^T
\log q_\theta(x_t\mid x_{<t})
=
\log Q_\theta(x).
\end{align}
Hence
\begin{align}
\mathcal{L}(\theta)
=
\mathbb E_{X\sim P}[-\log Q_\theta(X)]\\
=
-\sum_{x\in\mathcal X}P(x)\log Q_\theta(x).
\end{align}
Expanding the KL divergence gives
{\small
\begin{align}
\begin{aligned}
D_{\mathrm{KL}}(P\Vert Q_\theta)
&=
\sum_{x\in\mathcal X}
P(x)\log\frac{P(x)}{Q_\theta(x)}  \\
&=
\sum_{x\in\mathcal X}P(x)\log p(x) 
-
\sum_{x\in\mathcal X}P(x)\log Q_\theta(x) \\
&=
-\mathcal{H}(P)+L(\theta).
\end{aligned}
\end{align}
}
Therefore,
\begin{align}
\mathcal{L}(\theta)
=
\mathcal{H}(P)
+
D_{\mathrm{KL}}(P\Vert Q_\theta).
\end{align}
Since $\mathcal{H}(P)$ is independent of $\theta$, minimizing $\mathcal{L}(\theta)$ is equivalent to minimizing
$D_{\mathrm{KL}}(P\Vert Q_\theta)$.
\end{proof}

\paragraph{From the loss identity to optimization error.}
Theorem~\ref{thm:ntp_forward_kl} shows that minimizing the population
next-token loss is equivalent to minimizing
\(D_{\mathrm{KL}}(P_{\mathrm{data}}\Vert Q_\theta)\), since the entropy
\(\mathcal{H}(P_{\mathrm{data}})\) does not depend on \(\theta\). Therefore, any
suboptimality in the population loss transfers directly to suboptimality in
forward KL. This yields the following corollary, which separates the learned
model's KL error into an optimization error and the best achievable KL within
the model class.

\subsection{Proof of Corollary~\ref{col:epsilonoptimal}}
\begin{corollary} 
\label{col:epsilonoptimal}
If $\hat\theta$ is an $\varepsilon$-optimal solution for the population next-token loss, i.e.,
\begin{align}
\mathcal{L}(\hat\theta)
\le
\inf_{\theta\in\Theta}\mathcal{L}(\theta)+\varepsilon,
\end{align}
then
\begin{equation}
\boxed{
D_{\mathrm{KL}}
\left(
P_{\mathrm{data}}\Vert Q_{\hat\theta}
\right)
\le
\inf_{\theta\in\Theta}
D_{\mathrm{KL}}(P\Vert Q_\theta)+\varepsilon.
}
\end{equation}
Here, \(\epsilon\) captures optimization error, which depends on the optimization landscape and the optimizer. The term
\begin{align}
\inf_{\theta\in\Theta}D_{\mathrm{KL}}(P_{\mathrm{data}}\|Q_\theta)
\end{align}
is the model-class mispecification gap, governed by the expressive power of the model class.
\end{corollary}

\begin{proof}
Using the same identity,
\begin{align}
\inf_{\theta\in\Theta}L(\theta)
=
\mathcal{H}(P)
+
\inf_{\theta\in\Theta}
D_{\mathrm{KL}}(P\Vert Q_\theta)
=
\mathcal{H}(p)+D^*.
\end{align}
Thus, if $\mathcal{L}(\hat\theta)\le \inf_{\theta}L(\theta)+\varepsilon$, then
\begin{align}
\begin{aligned}
D_{\mathrm{KL}}(P\Vert Q_{\hat\theta})
&=
\mathcal{L}(\hat\theta)-\mathcal{H}(P) \\
&\le
\inf_{\theta}\mathcal{L}(\theta)+\varepsilon-\mathcal{H}(P) \\
&=
D^*+\varepsilon.
\end{aligned}
\end{align}
Substituting $P=P_{\mathrm{data}}$ and $Q_{\hat\theta}=P_{\hat\theta}$ completes the proof.
\end{proof}
The next subsection shows that, for a sufficiently expressive Transformer
logit class, this misspecification gap can be made arbitrarily small by applying
universal approximation to the smoothed true next-token logit map.

\subsection{KL-bound via Universal Approximation Theorem. }
\label{appen:ua-to-kl}
The preceding corollary reduces the statistical approximation question to
bounding the model-class misspecification gap
\begin{align}
    \inf_{\theta\in\Theta}
    D_{\mathrm{KL}}(P_{\mathrm{data}}\Vert P_\theta).
\end{align}
We now show that this gap vanishes for a sufficiently expressive Transformer
logit class. The proof applies universal approximation to the target
next-token logit map rather than directly to probabilities. Since the true
conditionals may assign zero probability to some tokens, we first smooth them;
this makes the target logits finite and compatible with the softmax model,
which always assigns positive probability to every vocabulary element.

Let \(\mathcal V\) be a finite vocabulary with \(M=|\mathcal V|\), and let
\(P_{\mathrm{data}}\in\Delta(\mathcal V^T)\) be the pretraining distribution
over length-\(T\) token sequences \(X=(X_1,\ldots,X_T)\). For \(t\in[T]\), write
\(X_{<t}=(X_1,\ldots,X_{t-1})\). For each data-supported prefix
\(h\in\mathcal V^{t-1}\), define the true next-token conditional
\begin{equation}
    q_t(v\mid h)
    :=
    P_{\mathrm{data}}(X_t=v\mid X_{<t}=h),
    \qquad v\in\mathcal V .
    \label{eq:true-conditional-token}
\end{equation}
For unsupported prefixes, \(q_t(\cdot\mid h)\) may be defined arbitrarily, since
such prefixes have zero probability under \(P_{\mathrm{data}}\).

The model is autoregressive:
\begin{align}
    P_\theta(x_{1:T})
    =
    \prod_{t=1}^T r_{t,\theta}(x_t\mid x_{<t}),
    \\
    r_{t,\theta}(\cdot\mid h)
    =
    \mathrm{softmax}\bigl(\widehat\ell_{t,\theta}(h)\bigr),
    \label{eq:model-token-autoregressive}
\end{align}
where \(\widehat\ell_{t,\theta}(h)\in\mathbb R^M\) denotes the model logit
vector. Let
\begin{align}
    \mathcal P_T^{\mathrm{supp}}
    :=
    \{(t,h): t\in[T],\ h\in\mathcal V^{t-1}, \\
    P_{\mathrm{data}}(X_{<t}=h)>0\}
\end{align}
denote the set of data-supported prefixes.

For \(\alpha\in(0,1)\), define the smoothed conditional and its logit map by
\begin{align}
    q_{t,\alpha}(v\mid h)
    :=
    (1-\alpha)q_t(v\mid h)
    +
    \frac{\alpha}{M},\\
    U_\alpha(t,h)
    :=
    \bigl(\log q_{t,\alpha}(v\mid h)\bigr)_{v\in\mathcal V} \in \mathbb{R}^M.
    \label{eq:smoothed-target-logits}
\end{align}
Since \(q_{t,\alpha}(v\mid h)\geq \alpha/M\), the target logits
\(U_\alpha(t,h)\) are finite for all \(v\in\mathcal V\).

\paragraph{Finite-prefix logit approximation.}
For fixed vocabulary size \(M\) and sequence length \(T\), the set of
data-supported prefixes \(\mathcal P_T^{\mathrm{supp}}\) is finite. After token
and positional embeddings, the smoothed target logit map \(U_\alpha\) is
therefore a finite table of bounded vectors in \(\mathbb R^M\). By extending
this finite table to a continuous function on a compact neighborhood, the
Transformer universal approximation theorem with positional encodings motivates
the uniform finite-prefix logit approximation assumption used below.
\begin{lemma}[Uniform sequence-to-vector Transformer approximation]
\label{lem:uniform-seq-to-vector}
Let $d,T,m<\infty$ and let
\begin{align}
\mathcal K=[0,1]^{d\times T}.
\end{align}
Under the sequence-to-vector Transformer approximation framework of
\cite{yu2026attentionheadcount}, for every continuous target
$F\in C(\mathcal K;\mathbb R^m)$ and every $\varepsilon>0$, there exists a
sequence-to-vector Transformer map
\begin{align}
G_\theta:\mathcal K\to\mathbb R^m
\end{align}
with sufficiently large parameter budget such that
\begin{align}
\sup_{X\in\mathcal K}
\|G_\theta(X)-F(X)\|_\infty
<
\varepsilon.
\end{align}
\end{lemma}
\begin{proof}
The proof is a direct consequence of the approximation results of
\cite{yu2026attentionheadcount}. Their Definition~1 uses the uniform
approximation criterion
\begin{align}
\sup_{X_T\in\mathcal X_T}
\|\widehat H(X_T)-F(X_T)\|_\infty<\varepsilon.
\end{align}
Their Theorem~1 proves that the generalized $D$-retrieval target family is
dense in the space of continuous sequence-to-vector functions. Hence, for any
coordinate function $F_j$ of $F$ and any $\varepsilon>0$, there exists a
generalized retrieval target $F_{D,j}$ such that
\begin{align}
\sup_{X\in\mathcal K}|F_j(X)-F_{D,j}(X)|
<
\frac{\varepsilon}{2}.
\end{align}
The vector-valued case is handled coordinate-wise, by stacking the coordinate
approximants into a vector-valued generalized retrieval target $F_D$ satisfying
\begin{align}
\sup_{X\in\mathcal K}
\|F(X)-F_D(X)\|_\infty
<
\frac{\varepsilon}{2}.
\end{align}
Their Transformer approximation theorem then gives a sequence-to-vector
Transformer $G_\theta$ with sufficiently large parameter budget such that
\begin{align}
\sup_{X\in\mathcal K}
\|G_\theta(X)-F_D(X)\|_\infty
<
\frac{\varepsilon}{2}.
\end{align}
By the triangle inequality,
\begin{align}
\sup_{X\in\mathcal K}
\|G_\theta(X)-F(X)\|_\infty
<
\varepsilon.
\end{align}
This proves the claim.
\end{proof}

\begin{proposition}[Uniform approximation of smoothed finite-prefix logits]
\label{prop:finite-prefix-logit-approx}
Assume that the autoregressive Transformer logit class contains, or can
simulate, the sequence-to-vector Transformer class in
Lemma~\ref{lem:uniform-seq-to-vector} on padded prefix inputs. Then, for every
fixed $\alpha\in(0,1)$ and every $\eta>0$, there exists a Transformer logit map
$\ell_{t,\theta}(h)\in\mathbb R^{M}$ such that
\begin{equation}
\max_{(t,h)\in\mathcal P_T^{\rm supp}}
\max_{v\in\mathcal V}
\left|
\ell^v_{t,\theta}(h)
-
\log q_{t,\alpha}(v\mid h)
\right|
<\eta.
\label{eq:uniform-logit-approx}
\end{equation}
\end{proposition}

\begin{proof}
Because $\mathcal V$ and $T$ are finite, the supported-prefix set
$\mathcal P_T^{\rm supp}$ is finite. Thus $U_\alpha$ is a bounded finite table
on $\mathcal P_T^{\rm supp}$.

We now encode each prefix as a fixed-length sequence in a compact Euclidean
domain. Let $\bot\notin\mathcal V$ be a padding symbol, and choose an injective
one-hot embedding
\begin{align}
e:\mathcal V\cup\{\bot\}
\to
\{0,1\}^{M+1}
\subset [0,1]^{M+1}.
\end{align}
For a supported prefix $(t,h)\in\mathcal P_T^{\rm supp}$, with
$h=(h_1,\ldots,h_{t-1})$, define the padded prefix encoding
\begin{align}
\iota(t,h)
:=
(x_1,\ldots,x_T)
\in [0,1]^{(M+1)\times T}
\end{align}
by
\begin{align}
x_j
=
\begin{cases}
e(h_j), & j<t,\\
e(\bot), & j\ge t.
\end{cases}
\end{align}
The map $\iota$ is injective: the first padding position identifies $t$, and
the preceding non-padding entries identify $h$.

Define
\begin{align}
\mathcal K
:=
[0,1]^{(M+1)\times T},\\
\mathcal Z
:=
\{\iota(t,h):(t,h)\in\mathcal P_T^{\rm supp}\}.
\end{align}
Then $\mathcal Z$ is a finite subset of the compact domain $\mathcal K$.
Define a finite target table on $\mathcal Z$ by
\begin{align}
Y(\iota(t,h)):=U_\alpha(t,h).
\end{align}

We next extend this finite table to a continuous function on $\mathcal K$.
Write $\mathcal Z=\{z_1,\ldots,z_N\}$ and $y_i:=Y(z_i)$. If $N\ge 2$, choose
\begin{align}
0<\rho<
\frac12
\min_{i\ne j}\|z_i-z_j\|_\infty.
\end{align}
If $N=1$, choose any $\rho>0$. For each $i\in[N]$, define the continuous bump
function
\begin{align}
\varphi_i(x)
:=
\max\left\{
0,\,
1-\frac{\|x-z_i\|_\infty}{\rho}
\right\},
\qquad x\in\mathcal K.
\end{align}
Then
\begin{align}
\varphi_i(z_i)=1,
\qquad
\varphi_i(z_j)=0
\quad\text{for } j\ne i.
\end{align}
Now define
\begin{align}
\widetilde U_\alpha(x)
:=
\sum_{i=1}^N y_i\,\varphi_i(x),
\qquad x\in\mathcal K.
\end{align}
Then
\begin{align}
\widetilde U_\alpha\in C(\mathcal K;\mathbb R^{M})
\end{align}
and, for every supported prefix $(t,h)\in\mathcal P_T^{\rm supp}$,
\begin{align}
\widetilde U_\alpha(\iota(t,h))
=
U_\alpha(t,h).
\end{align}

By Lemma~\ref{lem:uniform-seq-to-vector}, applied with
$d=M+1$ and output dimension $m=M$, there exists a
sequence-to-vector Transformer
\begin{align}
G_\theta:\mathcal K\to\mathbb R^{M}
\end{align}
such that
\begin{align}
\sup_{x\in\mathcal K}
\|G_\theta(x)-\widetilde U_\alpha(x)\|_\infty
<
\eta.
\end{align}
Define the autoregressive next-token logit map by evaluating this Transformer
on the padded prefix:
\begin{align}
\ell_{t,\theta}(h)
:=
G_\theta(\iota(t,h)).
\end{align}
Then, for every $(t,h)\in\mathcal P_T^{\rm supp}$,
\begin{align}
\begin{aligned}
\max_{v\in\mathcal V}
\left|
\ell^v_{t,\theta}(h)
-
\log q_{t,\alpha}(v\mid h)
\right| \\
=
\left\|
G_\theta(\iota(t,h))
-
U_\alpha(t,h)
\right\|_\infty\\
=
\left\|
G_\theta(\iota(t,h))
-
\widetilde U_\alpha(\iota(t,h))
\right\|_\infty
<\eta.
\end{aligned}
\end{align}
Taking the maximum over all supported prefixes gives
Eq.~\eqref{eq:uniform-logit-approx}.
\end{proof}

\begin{theorem}[KL control from universal logit approximation]

Let
\begin{align}
r_{t,\theta}(\cdot\mid h)
:=
\operatorname{softmax}(\ell_{t,\theta}(h))
\end{align}
and define the autoregressive model
\begin{align}
P_\theta(x_{1:T})
=
\prod_{t=1}^T r_{t,\theta}(x_t\mid x_{<t}).
\end{align}
Assume that the uniform logit approximation in Eq.~\eqref{eq:uniform-logit-approx} holds; that is,
\begin{align}
\max_{(t,h)\in\mathcal P_T^{\mathrm{supp}}}\max_{v\in V}
\left|\ell^v_{t,\theta}(h)-\log q_{t,\alpha}(v\mid h)\right|
\le \eta .
\end{align}
Then
\begin{align}
D_{\rm KL}(P_{\rm data}\|P_\theta)
\le
T\left[-\log(1-\alpha)+2\eta\right].
\end{align}
Consequently, for every $\varepsilon_{\rm KL}>0$, choosing
\begin{align}
\alpha
=
1-\exp\left(-\frac{\varepsilon_{\rm KL}}{2T}\right),
\qquad
\eta
=
\frac{\varepsilon_{\rm KL}}{4T}
\end{align}
gives
\begin{align}
D_{\rm KL}(P_{\rm data}\|P_\theta)
\le
\varepsilon_{\rm KL}.
\end{align}
\end{theorem}

\begin{proof}
Fix \(t\in[T]\) and a data-supported prefix
\(h\in\mathcal V^{t-1}\). To simplify notation, write
\begin{align}
    q(v) := q_t(v\mid h),\\
    q_\alpha(v) := q_{t,\alpha}(v\mid h),\\
    r(v) := r_{t,\theta}(v\mid h).
\end{align}
Let
\begin{align}
    \ell_\alpha^v := \log q_\alpha(v),
    \qquad
    \widehat \ell^v := \widehat \ell_{t,\theta}^v(h).
\end{align}
By the uniform logit approximation assumption,
\begin{align}
    \widehat \ell^v
    =
    \ell_\alpha^v
    +
    \Delta_v,
    \qquad
    |\Delta_v|\leq \eta
    \quad
    \text{for all }v\in\mathcal V.
\end{align}
Since \(q_\alpha=\mathrm{softmax}(\ell_\alpha)\), the model distribution can be
written as
\begin{align}
    r(v)
    &=
    \frac{
        \exp(\widehat \ell^v)
    }{
        \sum_{u\in\mathcal V}
        \exp(\widehat \ell^u)
    }                                                     \nonumber\\
    &=
    \frac{
        q_\alpha(v)\exp(\Delta_v)
    }{
        \sum_{u\in\mathcal V}
        q_\alpha(u)\exp(\Delta_u)
    }.
    \label{eq:r-in-terms-of-qalpha}
\end{align}
Define
\begin{align}
    Z_\Delta
    :=
    \sum_{u\in\mathcal V}
    q_\alpha(u)\exp(\Delta_u).
\end{align}
Then
\begin{align}
    r(v)
    =
    \frac{
        q_\alpha(v)\exp(\Delta_v)
    }{
        Z_\Delta
    }.
\end{align}

We first bound the conditional KL at the fixed prefix \(h\):
\begin{align}
    D_{\mathrm{KL}}(q\|r)
    =
    \sum_{v\in\mathcal V}
    q(v)
    \log
    \frac{q(v)}{r(v)}                                                    \nonumber\\
    =
    \sum_{v\in\mathcal V}
    q(v)
    \log
    \frac{q(v)}{q_\alpha(v)}
    +
    \sum_{v\in\mathcal V}
    q(v)
    \log
    \frac{q_\alpha(v)}{r(v)}                                             \nonumber\\
    =
    D_{\mathrm{KL}}(q\|q_\alpha)
    +
    \sum_{v\in\mathcal V}
    q(v)
    \log
    \frac{q_\alpha(v)}{r(v)} .
    \label{eq:conditional-kl-split}
\end{align}

We bound the two terms separately. For the smoothing term, observe that
\begin{align}
    q_\alpha(v)
    =
    (1-\alpha)q(v)
    +
    \frac{\alpha}{M}
    \geq
    (1-\alpha)q(v).
\end{align}
Therefore, whenever \(q(v)>0\),
\begin{align}
    \frac{q(v)}{q_\alpha(v)}
    \leq
    \frac{1}{1-\alpha}.
\end{align}
Hence
\begin{equation}
    D_{\mathrm{KL}}(q\|q_\alpha)
    \leq
    -\log(1-\alpha).
    \label{eq:smoothing-kl-bound}
\end{equation}

For the second term in \eqref{eq:conditional-kl-split}, using
\eqref{eq:r-in-terms-of-qalpha},
\begin{align}
    \log\frac{q_\alpha(v)}{r(v)}
    =
    -\Delta_v
    +
    \log Z_\Delta.
\end{align}
Thus
\begin{align}
    \sum_{v\in\mathcal V}
    q(v)
    \log
    \frac{q_\alpha(v)}{r(v)}
    &=
    -
    \sum_{v\in\mathcal V}
    q(v)\Delta_v
    +
    \log Z_\Delta .
    \label{eq:logit-error-term}
\end{align}
Since \(|\Delta_v|\leq \eta\), we have
\begin{align}
    -
    \sum_{v\in\mathcal V}
    q(v)\Delta_v
    \leq
    \eta.
\end{align}
Also,
\begin{align}
    Z_\Delta
    =
    \sum_{u\in\mathcal V}
    q_\alpha(u)\exp(\Delta_u)
    \leq
    \sum_{u\in\mathcal V}
    q_\alpha(u)e^\eta
    =
    e^\eta,
\end{align}
so
\begin{align}
    \log Z_\Delta
    \leq
    \eta.
\end{align}
Substituting these two inequalities into \eqref{eq:logit-error-term} gives
\begin{equation}
    \sum_{v\in\mathcal V}
    q(v)
    \log
    \frac{q_\alpha(v)}{r(v)}
    \leq
    2\eta.
    \label{eq:logit-error-kl-bound}
\end{equation}
Combining \eqref{eq:smoothing-kl-bound} and \eqref{eq:logit-error-kl-bound},
we obtain, for every data-supported prefix \(h\),
\begin{equation}
    D_{\mathrm{KL}}
    \left(
        q_t(\cdot\mid h)
        \,\middle\|\,
        r_{t,\theta}(\cdot\mid h)
    \right)
    \leq
    -\log(1-\alpha)
    +
    2\eta .
    \label{eq:conditional-kl-final-bound}
\end{equation}

Now apply the KL chain rule for autoregressive distributions:
\begin{multline}
    D_{\mathrm{KL}}
    \left(
        P_{\mathrm{data}}
        \,\middle\|\,
        P_\theta
    \right)
    \\
    =
    \sum_{t=1}^{T}
    \mathbb E_{X\sim P_{\mathrm{data}}}
    \Bigl[                                                        
        D_{\mathrm{KL}}
        \left(
            q_t(\cdot\mid X_{<t})
            \,\middle\|\,
            r_{t,\theta}(\cdot\mid X_{<t})
        \right)
    \Bigr]                                                        \\
    \leq
    \sum_{t=1}^{T}
    \left(
        -\log(1-\alpha)
        +
        2\eta
    \right)                                                        \\
    =
    T
    \left(
        -\log(1-\alpha)
        +
        2\eta
    \right).
    \label{eq:chain-rule-final-bound}
\end{multline}
This proves \eqref{eq:kl-bound-alpha-eta}.

Finally, choose
\begin{align}
    \alpha
    =
    1-\exp\left(
        -\frac{\varepsilon_{\mathrm{KL}}}{2T}
    \right),
    \qquad
    \eta
    =
    \frac{\varepsilon_{\mathrm{KL}}}{4T}.
\end{align}
Then
\begin{align}
    -\log(1-\alpha)
    =
    \frac{\varepsilon_{\mathrm{KL}}}{2T},
    \qquad
    2\eta
    =
    \frac{\varepsilon_{\mathrm{KL}}}{2T}.
\end{align}
Therefore,
\begin{align}
    D_{\mathrm{KL}}
    \left(
        P_{\mathrm{data}}
        \,\middle\|\,
        P_\theta
    \right)
    \leq
    T
    \left(
        \frac{\varepsilon_{\mathrm{KL}}}{2T}
        +
        \frac{\varepsilon_{\mathrm{KL}}}{2T}
    \right)
    =
    \varepsilon_{\mathrm{KL}}.
\end{align}
This completes the proof.
\end{proof}

\paragraph{Why universal approximation gives the logit approximation.}
The prefix set \(\mathcal P_T^{\mathrm{supp}}\) is finite. After embedding tokens
and positions into a Euclidean space, the target map \(U_\alpha\) is a finite table of bounded vectors in
\(\mathbb R^M\). Since finite-domain functions can be extended to continuous
functions on compact neighborhoods, the transformer universal approximation
theorem can be applied to this continuous extension. Therefore, for any
\(\eta>0\), there exists a transformer with a vocabulary logit head satisfying
\eqref{eq:uniform-logit-approx}. Theorem~\ref{thm:kl-control} then implies that
the induced autoregressive model can make
\begin{align}
    D_{\mathrm{KL}}
    \left(
        P_{\mathrm{data}}
        \,\middle\|\,
        P_\theta
    \right)
\end{align}
arbitrarily small.

\paragraph{Full-support special case.}
If every true conditional has full support, i.e.,
\(q_t(v\mid h)>0\) for all data-supported prefixes \(h\) and all
\(v\in\mathcal V\), then smoothing is unnecessary. One may set
\begin{align}
    H(t,h)
    =
    \left(
        \log q_t(v\mid h)
    \right)_{v\in\mathcal V}.
\end{align}
If the transformer approximates this unsmoothed logit map with error \(\eta\),
then the same proof with \(\alpha=0\) gives
\begin{align}
    D_{\mathrm{KL}}
    \left(
        P_{\mathrm{data}}
        \,\middle\|\,
        P_\theta
    \right)
    \leq
    2T\eta.
\end{align}
Thus choosing \(\eta=\varepsilon_{\mathrm{KL}}/(2T)\) yields
\begin{align}
    D_{\mathrm{KL}}
    \left(
        P_{\mathrm{data}}
        \,\middle\|\,
        P_\theta
    \right)
    \leq
    \varepsilon_{\mathrm{KL}}.
\end{align}

\paragraph{Remark on the approximation mode.}
The KL argument above requires the uniform logit approximation
\eqref{eq:uniform-logit-approx} on all data-supported prefixes. If one invokes
a universal approximation theorem stated only in an \(L^p\) metric over a
continuous domain, that statement alone does not control pointwise errors on
individual token prefixes. In that case, one should either use a uniform
approximation version, or explicitly extend the finite prefix table to small
positive-measure neighborhoods and approximate uniformly on those neighborhoods.

\paragraph{From qualitative approximation to rates.}
The previous theorem is qualitative: it shows that the pretraining KL can be
made arbitrarily small whenever the Transformer logit class can approximate the
smoothed target logits uniformly on supported prefixes. It does not quantify how
large the architecture must be to achieve a prescribed KL error. We next refine
this statement by replacing the abstract logit error \(\eta\) with an explicit
Transformer approximation rate. This yields a rate for the autoregressive KL in
terms of the attention approximation budget and the feed-forward approximation
budget.

\subsection{Approximation Rate for the Autoregressive KL}
\label{appen:kl-approx-rate}

We now refine Theorem~\ref{thm:kl-control} by incorporating a Jackson-type
approximation rate for Transformers. The approximation-rate result of
\cite{jiang2024approximation} shows that, for targets with suitable
Transformer-adapted complexity, the approximation error decomposes into an
attention low-rank term and a feed-forward approximation term.

To avoid notational conflict with the smoothing parameter, we use \(a>1/2\)
for the decay exponent of the temporal coupling component, and \(b>0\) for the
feed-forward approximation exponent.

For \(\alpha\in(0,1)\), define the smoothed true conditional
\begin{equation}
    q_{t,\alpha}(v\mid h)
    :=
    (1-\alpha)q_t(v\mid h)
    +
    \frac{\alpha}{M},
    \qquad v\in\mathcal V.
    \label{eq:smoothed-conditional-rate}
\end{equation}
The corresponding smoothed target logit map is
\begin{equation}
    U_\alpha^v(t,h)
    :=
    \log q_{t,\alpha}(v\mid h),
    \qquad v\in\mathcal V.
    \label{eq:smoothed-target-logit-rate}
\end{equation}

\paragraph{Jackson-type logit approximation assumption.}
For each vocabulary coordinate \(v\in\mathcal V\), assume that the scalar
sequence-to-sequence target
\begin{align}
    U_\alpha^v
    =
    \{H_{\alpha,t}^v\}_{t=1}^T
\end{align}
belongs to the Transformer approximation space \(C^{(a,b)}\) of
\cite{jiang2024approximation}. Let \(m_h\) denote the hidden dimension of the
attention mechanism, and let \(m_{\mathrm{FF}}\) denote the approximation budget
of the feed-forward networks.

Following the Jackson-type rate, define
\begin{multline}
    \mathcal R_\alpha^v(m_h,m_{\mathrm{FF}})\\
    :=
    T^2 C_{0,\alpha}^v
    \left(
        \frac{C_{1,\alpha}^v}{m_h^{2a-1}}
        +
        \frac{
            C_{2,\alpha}^v(m_h)m_h^{b+1}
        }{
            m_{\mathrm{FF}}^b
        }
    \right),
    \label{eq:coordinate-rate}
\end{multline}
where \(C_{0,\alpha}^v\), \(C_{1,\alpha}^v\), and
\(C_{2,\alpha}^v(m_h)\) are the corresponding complexity measures of the
target logit coordinate \(U_\alpha^v\). Define the worst-coordinate rate
\begin{equation}
    \mathcal R_\alpha(m_h,m_{\mathrm{FF}})
    :=
    \max_{v\in\mathcal V}
    \mathcal R_\alpha^v(m_h,m_{\mathrm{FF}}).
    \label{eq:worst-coordinate-rate}
\end{equation}

Since our KL objective is defined on discrete token prefixes, we assume the
Transformer logit class realizes the above Jackson rate uniformly on
data-supported prefixes:
\begin{equation}
    \max_{(t,h)\in\mathcal P_T^{\mathrm{supp}}}
    \max_{v\in\mathcal V}
    \left|
        \widehat \ell_{t,\theta}^v(h)
        -
        U_\alpha^v(t,h)
    \right|
    \leq
    \mathcal R_\alpha(m_h,m_{\mathrm{FF}}).
    \label{eq:uniform-jackson-logit-rate}
\end{equation}

\begin{theorem}[Jackson-type approximation rate for autoregressive KL]
\label{thm:kl-rate-main}
Suppose
that the smoothed target logit map \(U_\alpha\) satisfies the uniform
Jackson-type logit approximation rate
\eqref{eq:uniform-jackson-logit-rate}. Then the autoregressive model
\(P_\theta\) satisfies
{\small
\begin{equation}
    D_{\mathrm{KL}}
    \left(
        P_{\mathrm{data}}
        \,\middle\|\,
        P_\theta
    \right)
    \leq
    T
    \left(
        -\log(1-\alpha)
        +
        2\mathcal R_\alpha(m_h,m_{\mathrm{FF}})
    \right).
    \label{eq:kl-rate-main}
\end{equation}
}
Equivalently,
\begin{multline}
    D_{\mathrm{KL}}
    \left(
        P_{\mathrm{data}}
        \,\middle\|\,
        P_\theta
    \right)
    \\
    \leq
    T
    \Biggl[
        -\log(1-\alpha)
        \\
        {}+
        2
        \max_{v\in\mathcal V}
        T^2 C_{0,\alpha}^v
        \\
        {}\cdot
        \left(
            \frac{C_{1,\alpha}^v}{m_h^{2a-1}}
            +
            \frac{
                C_{2,\alpha}^v(m_h)m_h^{b+1}
            }{
                m_{\mathrm{FF}}^b
            }
        \right)
    \Biggr].
    \label{eq:expanded-kl-rate}
\end{multline}
\end{theorem}

\begin{proof}
Fix a time \(t\in[T]\) and a data-supported prefix
\(h\in\mathcal V^{t-1}\). Write
\begin{align}
    q(v):=q_t(v\mid h),\\
    q_\alpha(v):=q_{t,\alpha}(v\mid h),\\
    r(v):=r_{t,\theta}(v\mid h).
\end{align}
Let
\begin{align}
    \ell_\alpha^v := \log q_\alpha(v),
    \widehat \ell^v := \widehat \ell_{t,\theta}^v(h).
\end{align}
By \eqref{eq:uniform-jackson-logit-rate}, there exist errors
\(\Delta_v\) such that
\begin{align}
    \widehat \ell^v
    =
    \ell_\alpha^v
    +
    \Delta_v,\\
    |\Delta_v|
    \leq
    \mathcal R_\alpha(m_h,m_{\mathrm{FF}})
    \quad
    \forall v\in\mathcal V.
\end{align}
Since \(q_\alpha=\mathrm{softmax}(\ell_\alpha)\), the model distribution can
be written as
\begin{equation}
    r(v)
    =
    \frac{
        q_\alpha(v)\exp(\Delta_v)
    }{
        \sum_{u\in\mathcal V}
        q_\alpha(u)\exp(\Delta_u)
    }.
    \label{eq:model-rate-softmax-rewrite}
\end{equation}
Define
\begin{align}
    Z_\Delta
    :=
    \sum_{u\in\mathcal V}
    q_\alpha(u)\exp(\Delta_u).
\end{align}
Then
\begin{align}
    r(v)
    =
    \frac{
        q_\alpha(v)\exp(\Delta_v)
    }{
        Z_\Delta
    }.
\end{align}

We decompose the conditional KL:
\begin{multline}
    D_{\mathrm{KL}}(q\|r)
    =
    \sum_{v\in\mathcal V}
    q(v)\log\frac{q(v)}{r(v)}
    \nonumber\\
    =
    \sum_{v\in\mathcal V}
    q(v)\log\frac{q(v)}{q_\alpha(v)}
    +
    \sum_{v\in\mathcal V}
    q(v)\log\frac{q_\alpha(v)}{r(v)}
    \nonumber\\
    =
    D_{\mathrm{KL}}(q\|q_\alpha)
    +
    \sum_{v\in\mathcal V}
    q(v)\log\frac{q_\alpha(v)}{r(v)}.
    \label{eq:conditional-kl-rate-split}
\end{multline}

First, because
\begin{align}
    q_\alpha(v)
    =
    (1-\alpha)q(v)
    +
    \frac{\alpha}{M}
    \geq
    (1-\alpha)q(v),
\end{align}
we have
\begin{equation}
    D_{\mathrm{KL}}(q\|q_\alpha)
    \leq
    -\log(1-\alpha).
    \label{eq:smoothing-rate-bound}
\end{equation}

Second, from \eqref{eq:model-rate-softmax-rewrite},
\begin{align}
    \log\frac{q_\alpha(v)}{r(v)}
    =
    -\Delta_v
    +
    \log Z_\Delta.
\end{align}
Therefore,
\begin{align}
    \sum_{v\in\mathcal V}
    q(v)\log\frac{q_\alpha(v)}{r(v)}
    &=
    -
    \sum_{v\in\mathcal V}
    q(v)\Delta_v
    +
    \log Z_\Delta.
    \label{eq:rate-error-term}
\end{align}
Since
\begin{align}
    |\Delta_v|
    \leq
    \mathcal R_\alpha(m_h,m_{\mathrm{FF}}),
\end{align}
we get
\begin{align}
    -
    \sum_{v\in\mathcal V}
    q(v)\Delta_v
    \leq
    \mathcal R_\alpha(m_h,m_{\mathrm{FF}}).
\end{align}
Also,
\begin{align}
    Z_\Delta
    =
    \sum_{u\in\mathcal V}
    q_\alpha(u)\exp(\Delta_u)
    \leq
    \exp\left(
        \mathcal R_\alpha(m_h,m_{\mathrm{FF}})
    \right),
\end{align}
so
\begin{align}
    \log Z_\Delta
    \leq
    \mathcal R_\alpha(m_h,m_{\mathrm{FF}}).
\end{align}
Thus,
\begin{equation}
    \sum_{v\in\mathcal V}
    q(v)\log\frac{q_\alpha(v)}{r(v)}
    \leq
    2\mathcal R_\alpha(m_h,m_{\mathrm{FF}}).
    \label{eq:logit-rate-kl-bound}
\end{equation}

Combining \eqref{eq:smoothing-rate-bound} and
\eqref{eq:logit-rate-kl-bound}, for every data-supported prefix \(h\),
\begin{multline}
    D_{\mathrm{KL}}
    \left(
        q_t(\cdot\mid h)
        \,\middle\|\,
        r_{t,\theta}(\cdot\mid h)
    \right)
    \leq
    -\log(1-\alpha) \\
    +
    2\mathcal R_\alpha(m_h,m_{\mathrm{FF}}).
    \label{eq:conditional-kl-rate-final}
\end{multline}

Using the autoregressive KL chain rule,
\begin{multline}
    D_{\mathrm{KL}}
    \left(
        P_{\mathrm{data}}
        \,\middle\|\,
        P_\theta
    \right)
    \\
    =
    \sum_{t=1}^{T}
    \mathbb E_{X\sim P_{\mathrm{data}}}
    \Bigl[                                                        
        D_{\mathrm{KL}}
        \left(
            q_t(\cdot\mid X_{<t})
            \,\middle\|\,
            r_{t,\theta}(\cdot\mid X_{<t})
        \right)
    \Bigr]                                                        \\
    \leq
    \sum_{t=1}^{T}
    \left(
        -\log(1-\alpha)
        +
        2\mathcal R_\alpha(m_h,m_{\mathrm{FF}})
    \right)                                                       \\
    =
    T
    \left(
        -\log(1-\alpha)
        +
        2\mathcal R_\alpha(m_h,m_{\mathrm{FF}})
    \right).
\end{multline}
This proves \eqref{eq:kl-rate-main}. Substituting the definition of
\(\mathcal R_\alpha(m_h,m_{\mathrm{FF}})\) gives
\eqref{eq:expanded-kl-rate}.
\end{proof}

\paragraph{Interpretation.}
Theorem~\ref{thm:kl-rate-main} gives a quantitative version of the UAT-based
KL guarantee. The smoothing term \(T[-\log(1-\alpha)]\) is the price of making
the true next-token logits finite. The approximation term
\(2T\mathcal R_\alpha(m_h,m_{\mathrm{FF}})\) is the architectural
misspecification error. Thus, increasing the attention budget \(m_h\) and the
feed-forward budget \(m_{\mathrm{FF}}\) (i.e., a larger model with higher expressiveness) reduces the KL whenever the smoothed
target logits have finite Transformer-adapted complexity.

\section{Theoretical Motivation for Forgetting Risk and Adaptation Utility}
In this section, we provide additional analytical motivation for the definitions of forgetting risk and adaptation utility introduced in Sec.~\ref{sec:3.1}.

\paragraph{Definition of Forgetting Risk.}
Let
\(X_{\mathrm{pre}}\in\mathbb{R}^{d_{\mathrm{in}}\times T}\)
denote the layer-input token activations induced by pretraining data.
Consider
\begin{equation}
f_W(X_{\mathrm{pre}})
=
\sigma(WX_{\mathrm{pre}}+b\mathbf{1}_T^\top),
\end{equation}
where \(\sigma\) is applied element-wise. Let
\begin{equation}
Z_0=W_0X_{\mathrm{pre}}+b\mathbf{1}_T^\top,
\qquad
E=\Delta W X_{\mathrm{pre}}.
\end{equation}
Assume that \(\sigma\) is \(L_\sigma\)-Lipschitz on the line
segment between \(Z_0\) and \(Z_0+E\). Then
\begin{equation}
\small
\begin{aligned}
\left\|
f_{W_0+\Delta W}(X_{\mathrm{pre}})
-
f_{W_0}(X_{\mathrm{pre}})
\right\|_F^2
&=
\left\|
\sigma(Z_0+E)-\sigma(Z_0)
\right\|_F^2
\\
&\le
L_\sigma^2\|E\|_F^2
\\
&=
L_\sigma^2
\|\Delta W X_{\mathrm{pre}}\|_F^2.
\end{aligned}
\end{equation}
Moreover,
{\small
\begin{equation}
\|\Delta W X_{\mathrm{pre}}\|_F^2
=
\operatorname{Tr}
\left(
\Delta W C_{\mathrm{pre}}\Delta W^\top
\right),
\quad
C_{\mathrm{pre}}
:=
X_{\mathrm{pre}}X_{\mathrm{pre}}^\top.
\end{equation}
}
Therefore, we define the activation-space forgetting surrogate
{\small
\begin{equation}
\mathcal R_{\mathrm{pre}}(\Delta W)
:=
\|\Delta W X_{\mathrm{pre}}\|_F^2
=
\operatorname{Tr}
\left(
\Delta W C_{\mathrm{pre}}\Delta W^\top
\right).
\end{equation}
}
In particular, \(\mathcal R_{\mathrm{pre}}(\Delta W)=0\) implies
that the perturbed block produces the same output on
\(X_{\mathrm{pre}}\).

To connect this local quantity to network-level behavior, let
\(\Phi\) denote the fixed remainder of the network from the
considered block to the logits, and assume that \(\Phi\) is locally
\(L_\Phi\)-Lipschitz. Then
\begin{equation}
\small
\left\|
\Phi(f_{W_0+\Delta W}(X_{\mathrm{pre}}))
-
\Phi(f_{W_0}(X_{\mathrm{pre}}))
\right\|_F^2
\le
L_\Phi^2L_\sigma^2
\mathcal R_{\mathrm{pre}}(\Delta W).
\end{equation}
Therefore, \(\mathcal R_{\mathrm{pre}}\) is a tractable upper-bound
surrogate for the functional drift induced by \(\Delta W\) on
pretraining activations.

\paragraph{Definition of Adaptation Utility.}
Let
\begin{equation}
\mathcal J_{\mathrm{task}}(Z)
=
\ell(\Phi(\sigma(Z)),Y)
\end{equation}
denote the downstream task objective as a function of the
preactivation \(Z\) of the considered block, where \(\Phi\)
represents the fixed remainder of the network. Let
\begin{equation}
Z_0=W_0X_{\mathrm{task}}+b\mathbf{1}_T^\top,
\qquad
M=\Delta W X_{\mathrm{task}},
\end{equation}
and define
\begin{equation}
G_{\mathrm{task}}
=
\nabla_Z\mathcal J_{\mathrm{task}}(Z_0).
\end{equation}
The first-order task-loss variation induced by \(\Delta W\) is
\begin{equation}
\small
\mathcal J_{\mathrm{task}}(Z_0+M)
-
\mathcal J_{\mathrm{task}}(Z_0)
\approx
\left\langle
G_{\mathrm{task}},M
\right\rangle_F.
\end{equation}
By Frobenius duality,
\begin{equation}
\sup_{\|Q\|_F\le 1}
\left|
\left\langle Q,M\right\rangle_F
\right|^2
=
\|M\|_F^2.
\end{equation}
Therefore, \(\|M\|_F^2\) is the squared maximum first-order
response that the update can induce against a unit-norm
output-space task signal. Moreover,
\begin{equation}
\left|
\left\langle
G_{\mathrm{task}},M
\right\rangle_F
\right|^2
\le
\|G_{\mathrm{task}}\|_F^2\|M\|_F^2,
\end{equation}
so an update with small \(\|M\|_F^2\) cannot substantially
affect the task objective to first order.

Motivated by this characterization, we define the adaptation utility surrogate
\begin{equation}
\small
\mathcal U_{\mathrm{task}}(\Delta W)
:=
\|\Delta W X_{\mathrm{task}}\|_F^2
=
\operatorname{Tr}
\left(
\Delta W C_{\mathrm{task}}\Delta W^\top
\right),
\end{equation}
where 
$C_{\mathrm{task}}
:=
X_{\mathrm{task}}X_{\mathrm{task}}^\top.
$

\end{document}